\documentclass{article}

\usepackage{microtype}
\usepackage{amsmath}
\usepackage{graphicx}
\usepackage{booktabs} % for professional tables
\usepackage{algorithm}
\usepackage{graphicx}
\usepackage{bbm}
\usepackage{amsthm}   
\usepackage{amssymb}
\usepackage{caption}
\usepackage{wrapfig}
\usepackage{subcaption}
\usepackage{float}
\usepackage{array}
\usepackage{makecell}
\usepackage{enumitem}
\usepackage{cmll}
\usepackage{bbm}
\usepackage{amsfonts}       % blackboard math symbols
\usepackage{nicefrac}       % compact symbols for 1/2, etc.
\usepackage{microtype}      % microtypography
\usepackage{xcolor}         % colors
\usepackage{amsmath}

\definecolor{mydarkblue}{rgb}{0,0.08,0.45}
\definecolor{myred}{rgb}{0.84,0.17,0.11}
\definecolor{mygreen}{rgb}{0.35,0.60,0.25}
\definecolor{myblue}{rgb}{0.19,0.44,0.72}

\usepackage{hyperref}

\usepackage[accepted]{icml2024} 

\usepackage{amsmath}
\usepackage{amssymb}
\usepackage{mathtools}
\usepackage{amsthm}
\usepackage[thicklines]{cancel}
\usepackage{multirow}

\usepackage{xcolor}

\usepackage[capitalize,noabbrev]{cleveref}

\theoremstyle{plain}
\newtheorem{theorem}{Theorem}[section]
\newtheorem{proposition}[theorem]{Proposition}
\newtheorem{lemma}[theorem]{Lemma}

\theoremstyle{definition}

\theoremstyle{remark}

\usepackage[textsize=tiny]{todonotes}
\usepackage{amsmath,amsfonts,bm}

\def\eqref#1{equation~\ref{#1}}
\def\1{\bm{1}}

\def\rva{{\mathbf{a}}}

\def\rvs{{\mathbf{s}}}

\DeclareMathAlphabet{\mathsfit}{\encodingdefault}{\sfdefault}{m}{sl}
\SetMathAlphabet{\mathsfit}{bold}{\encodingdefault}{\sfdefault}{bx}{n}

\newcommand{\KL}{D_{\mathrm{KL}}}

\icmltitlerunning{Offline-Boosted Actor-Critic}

\begin{document}

\twocolumn[
\icmltitle{Offline-Boosted Actor-Critic: \\ Adaptively Blending Optimal Historical Behaviors in Deep Off-Policy RL}

\icmlsetsymbol{equal}{*}

\begin{icmlauthorlist}
\icmlauthor{Yu Luo}{thu}
\icmlauthor{Tianying Ji}{thu}
\icmlauthor{Fuchun Sun}{thu}
\icmlauthor{Jianwei Zhang}{hum}
\icmlauthor{Huazhe Xu}{cha,qzz,AIL}
\icmlauthor{Xianyuan Zhan}{AIL,air}
%\icmlauthor{}{sch}
%\icmlauthor{}{sch}
\end{icmlauthorlist}

\icmlaffiliation{thu}{Department of Computer Science and Technology, Tsinghua University}
\icmlaffiliation{hum}{Department of Informatics, University of Hamburg}
\icmlaffiliation{AIL}{Shanghai Artificial Intelligence Laboratory}
\icmlaffiliation{cha}{Institute for Interdisciplinary Information Sciences, Tsinghua University}
\icmlaffiliation{qzz}{Shanghai Qi Zhi Institute}
\icmlaffiliation{air}{Institute for AI Industry Research, Tsinghua University}

\icmlcorrespondingauthor{Fuchun Sun}{fcsun@tsinghua.edu.cn}

\icmlkeywords{Machine Learning, ICML}

\vskip 0.3in
]

\printAffiliationsAndNotice{}

\begin{abstract}
Off-policy reinforcement learning (RL) has achieved notable success in tackling many complex real-world tasks,
by leveraging previously collected data for policy learning. However, 
most existing off-policy RL algorithms fail to maximally exploit the information in the replay buffer, limiting sample efficiency and policy performance. In this work, we discover that concurrently training an offline RL policy based on the shared online replay buffer can sometimes outperform the original online learning policy, though the occurrence of such performance gains remains uncertain.
This motivates a new possibility of harnessing the emergent outperforming offline optimal policy to improve online policy learning.
Based on this insight, we present Offline-Boosted Actor-Critic (OBAC), a model-free online RL framework that elegantly identifies the outperforming offline policy through value comparison, and uses it as an adaptive constraint to guarantee stronger policy learning performance.
Our experiments demonstrate that OBAC outperforms other popular model-free RL baselines and rivals advanced model-based RL methods in terms of sample efficiency and asymptotic performance across \textbf{53} tasks spanning \textbf{6} task suites\footnote{Please refer to \textcolor{blue}{\url{https://roythuly.github.io/OBAC_web/}} for experiment videos and benchmark results.}.
\end{abstract}

\section{Introduction}
Online model-free deep reinforcement learning (RL) methods have achieved success in many challenging sequential 
\begin{figure}[H]
    \centering
    \includegraphics[width=\linewidth]{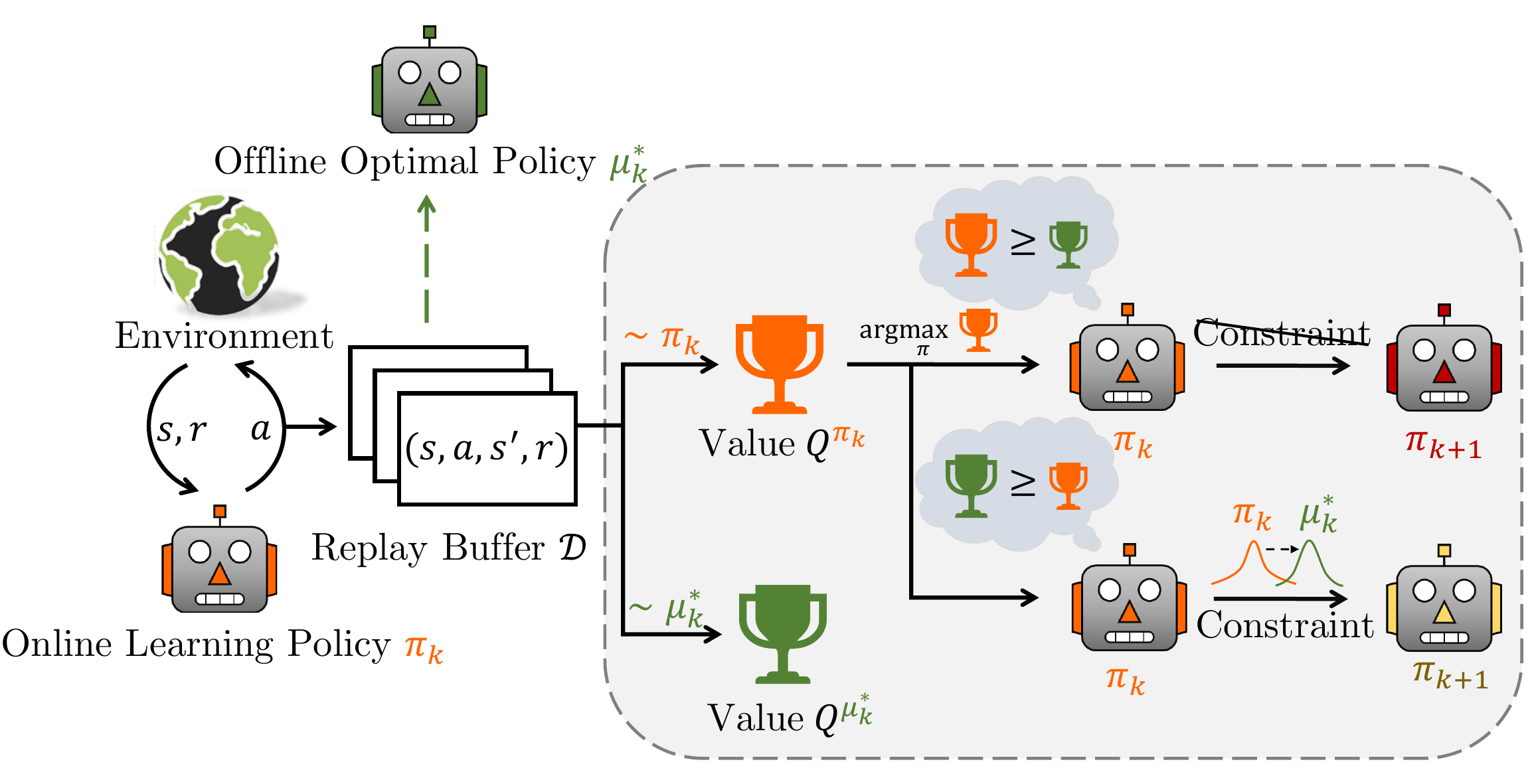}
    \includegraphics[width=\linewidth]{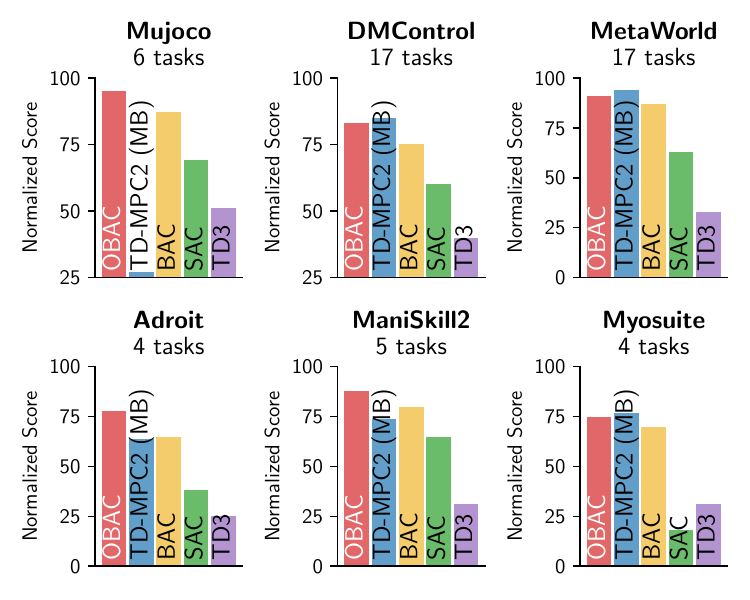}
    \caption{\textbf{Overview}. \textit{(Top)}: we illustrate the framework of OBAC, where the concurrent offline optimal policy can boost the online learning policy with an adaptive constraint mechanism. \textit{(Bottom)}: comparison of normalized score. Our OBAC can be comparable with advanced model-based RL method TD-MPCs, and outperform several popular model-free RL methods BAC, SAC and TD3.}
    \label{fig:overview}
\end{figure}
decision-making tasks~\cite{mnih2015human,van2016deep,wang2022deep}, 
including gaming AI~\cite{perolat2022mastering}, chip design~\cite{mirhoseini2021graph}, and automatic driving~\cite{kiran2021deep}. Many of these advances are attributed to off-policy RL methods~\cite{kallus2020double}, that enable agents to leverage collected data from historical policies to train the current policy. However, the reliance on millions of environment interaction steps still hampers the real-world deployment of RL~\cite{haarnoja2018soft}. 
We boil the algorithmic inefficiency down to their insufficient data utilization: 
when performing policy evaluation for value function learning and policy extraction via value maximization, 
most algorithms neglect and thus fail to leverage the inherent patterns and knowledge from the heterogeneous data in the replay buffer.

One cure for the inefficient data utilization of RL methods is model-based RL~\cite{moerland2023model}: learning an environmental dynamics model as the reservoir of domain knowledge, by generating new pseudo-samples for Dyna-style~\cite{ji2022update} or planning-style~\cite{hansen2023td} policy learning.
However, these approaches can be computationally complex and sometimes brittle due to the use of imperfect explicit model learning and long propagation chains.
Alternatively, offline RL~\cite{levine2020offline,kumar2020conservative,kostrikov2021offline} provides a new possibility by
allowing the learning of an optimal policy and the corresponding value function from fixed datasets without interacting with the environment. From an online RL perspective, leveraging such an offline learned policy and its offline value offers two advantages: \textit{(i)} the offline learned policy, a blend of optimal historical behaviors, can serve as an explicit performance baseline for current online policy optimization; and \textit{(ii)} the pessimistic training scheme~\cite{kostrikov2021offline,xu2022offline} enables in-distribution value estimation and offline policy learning, preventing bias propagation issues seen in model-based RL.

Several prior studies have explored the use of offline RL to enhance online off-policy RL training, generally falling into two directions, each exploiting key advantages of offline RL: \textit{(i)} incorporating an additional offline dataset for sampling augmentation~\cite{song2022hybrid,wagenmaker2023leveraging,ball2023efficient}, however, this may be expensive and easily impacted by data quality; and \textit{(ii)} without the offline dataset, learning an optimal offline value function from collected data (\emph{e.g.}, replay buffer) to adjust the online value function accordingly (\emph{e.g.}, via linear interpolation)~\cite{zhang2022replay,ji2023seizing,xu2023drm}. However, this approach may result in inaccurate policy evaluation, particularly when the offline policy in the replay buffer is of low quality, thus na\"ively mixing the offline value sometimes can be ineffective or even harmful.
Although we do observe in this work that the optimal offline policy learned from the replay buffer can often outperform the online policy, the occurrence of such superiority varies across tasks as well as different training stages. 
Yet, despite numerous previous attempts, a unified understanding and framework for leveraging offline RL for effective online off-policy RL
remains lacking.
This raises the following questions: \textit{When and how can we effectively leverage offline RL to ensure improvement in online off-policy RL?}

In this work, we introduce \textbf{O}ffline-\textbf{B}oosted \textbf{A}ctor-\textbf{C}ritic (\textbf{OBAC}) to address the above questions, providing a new solution to leverage offline RL for adaptively blending optimal historical behaviors in online off-policy RL, as shown in Figure~\ref{fig:overview}. To tackle the ``when'' issue,
we compare the evaluated state-values of the online learning policy and the offline optimal policy to identify the superior offline optimal policy.
To address the ``how'' challenge, we derive an adaptive mechanism that utilizes the superior offline optimal policy as a constraint to guide policy optimization.
In short,  OBAC can accumulate small performance gains when the offline optimal policy is better than the online learning policy, forming a positive cycle that a better offline policy helps online policy explore better data in return enhancing both policies in the next update, finally leading to significant overall performance improvement.
Notably, to circumvent the computational complexity \emph{w.r.t} explicitly learning the offline optimal policy—a similar issue seen in model-based RL when learning a dynamics model—we make a key technical contribution by introducing implicit offline policy learning in both evaluation and improvement steps, resulting in a cost-effective practical algorithm.

We evaluate our method across \textbf{53} diverse continuous control tasks spanning \textbf{6} domains: Mujoco~\cite{todorov2012mujoco}, DMControl~\cite{tassa2018deepmind}, Meta-World~\cite{yu2020meta}, Adroit~\cite{kumar2015mujoco}, Myosuite~\cite{caggiano2022myosuite}, and Maniskill2~\cite{gu2022maniskill2}, comparing it with BAC~\cite{ji2023seizing}, TD-MPC2~\cite{hansen2023td}, SAC~\cite{haarnoja2018soft}, and TD3~\cite{fujimoto2018addressing}. Our results, summarized in Figure 1, showcase OBAC's superiority. It outperforms BAC, the first documented effective model-free algorithm on challenging high-dimensional dog locomotion series tasks, by adjusting $Q$ values with offline values. When compared with TD-MPC2, a state-of-the-art model-based planning method
known for efficiency in various tasks, OBAC demonstrates comparable performance with only $\mathbf{50\%}$ of the parameters and $\mathbf{20\%}$ less training time\footnote{We have released our code here: \textcolor{blue}{\url{https://github.com/Roythuly/OBAC}}}. 

\section{Preliminaries}
We consider the conventional Markov Decision Process (MDP)~\cite{bellman1957markovian} defined by a 6-tuple $\mathcal{M}=\left\langle\mathcal{S},\mathcal{A},\mathcal{P},r,\gamma,d_0\right\rangle$, where $\mathcal{S}\in\mathbb{R}^n$ and $\mathcal{A}\in\mathbb{R}^m$ represent the continuous state and action spaces, $\mathcal{P}(s'|s,a):\mathcal{S}\times\mathcal{A}\rightarrow\Delta(\mathcal{S})$ denotes a Markovian transition (dynamics) distribution, $r(s,a):\mathcal{S}\times\mathcal{A}\rightarrow\Delta(\mathbb{R})$ is a stochastic reward function, $\gamma\in[0,1)$ gives the discounted factor for future rewards, and $d_0$ is the initial state distribution. The RL agent's objective is to find a policy $\pi(a|s):\mathcal{S}\rightarrow\Delta(\mathcal{A})$ that maximizes the discounted cumulative reward from the environment, $J_\pi=\mathbb{E}_{d_0,\pi,\mathcal{P}}\left[\sum_{t=0}^\infty\gamma^t r(s_t,a_t)\right]$.

We focus on the off-policy RL setting, where the agent interacts with the environment, collects new data into a replay buffer $\mathcal{D}\leftarrow\mathcal{D}\cup\{(s,a,s',r)\}$, and updates the learning policy using the stored data. At the $k$-th iteration step, the online learning policy is denoted as $\pi_k$, with its corresponding $Q$ value function
\begin{equation}
Q^{\pi_k}(s,a)=\mathbb{E}_{\pi_k,\mathcal{P}}\left[\sum_{t=0}^\infty\gamma^t r(s_t,a_t)|s_0=s,a_0=a\right]
\end{equation}
and the value function $V^{\pi_k}(s)=\mathbb{E}_{a\sim\pi_k}\left[Q^{\pi_k}(s,a)\right]$. 

Considering the replay buffer as a given dataset allows us to derive a concurrent offline optimal policy $\mu^*_k(a|s)$, given by
\begin{equation}\label{offline_policy_def}
\mu^*_k\triangleq\arg\max\mathbb{E}_{a\sim\mathcal{D}}\left[Q^{\mu_k}(s,a)\right].
\end{equation}
Unlike previous off-policy RL methods, we simultaneously train an online learning policy $\pi_k$ and an offline optimal policy $\mu^*_k$ by sharing a communal replay buffer $\mathcal{D}$. A key property of $\mu^*_k$ is its strong relevance to dataset distribution~\cite{fujimoto2019off,kumar2020conservative}, which means, action derived from it are restricted in support of actions in the replay buffer, $a\sim\mu^*_k\Rightarrow a\in\mathcal{D}$, while the online learning policy can have unrestricted action choices, $a\sim\pi_k\Rightarrow a\in\mathcal{A}$. Thus, $\mu^*_k$ characterizes the historical optimal behavior from the mixture of collected data, serving as an explicit and reliable performance baseline for $\pi_k$. Despite its conceptual simplicity, little previous work has introduced such a baseline for online learning optimization, resulting in wasted knowledge exploitation and sample inefficiency.

\section{Offline-Boosted Off-Policy RL}
In this section, we introduce Offline-Boosted Actor-Critic (OBAC), a framework aimed at improving the performance of online learning policies through the incorporation of an offline optimal policy. To gain insights into the behavior of such an offline optimal policy during online concurrent training, we conduct a set of thorough experiments, revealing its potential for outperforming the online learning policy. However, the timing of this superiority proves to be task-dependent and varies across different training stages. Taking this property into consideration, we detail how we adaptively integrate the offline optimal policy into alternating policy evaluation and optimization steps in off-policy RL paradigm. Following it up, we present a practical model-free RL algorithm based on the general actor-critic framework, achieving low computational cost and high sample efficiency.

\subsection{A Motivating Example}
In general, offline RL adheres to a principle of pessimism in policy training, aiming to prevent extrapolation errors in Q-value estimation by avoiding Out-Of-Distribution (OOD) actions~\cite{fujimoto2019off, xie2021bellman, shi2022pessimistic}. However, these concerns typically stem from training policies on fixed datasets. In contrast, during the online RL training process, as new samples are continuously collected through interaction with the environment, both the \textit{quantity} and \textit{quality} of data in the replay buffer grow dynamically. These characteristics break the limitations inherent in offline RL, prompting us to explore better performance by integrating historical optimal behavior as more data becomes available. To explore these properties, we conduct investigations across three OpenAI Gym environments (Hopper-v3, Walker2d-v3, Ant-v3), employing three different agents:

\noindent\textbf{Pure off-policy agent.}\quad
We train a Soft Actor-Critic (\textit{SAC}) agent~\cite{haarnoja2018soft} through $1$ million steps of environmental interaction. At each time step, new data is accumulated in a dedicated replay buffer $\mathcal{D}_{SAC}$, updating the online learning policy.

\noindent\textbf{Concurrent offline agent.}\quad
Simultaneously with training the SAC agent, we employ the Implicit Q-Learning (IQL)~\cite{kostrikov2021offline}, an offline RL algorithm, to learn an offline optimal policy concurrently within the dynamically changing SAC buffer $\mathcal{D}_{SAC}$, referred to \textit{IQL Concurrent}. Notably, the concurrent IQL agent does not interact with the environment during whole training process.

\noindent\textbf{Online-training offline agent.}\quad
In an online setting, we use IQL. This involves the IQL agent interacting with the environment, collecting new experiences stored in a replay buffer $\mathcal{D}_{IQL}$, and updating its policy, denoted as \textit{IQL Online}. The training procedure aligns with that of the SAC agent but employs a different algorithm.

\begin{figure}[t]
    \centering
    \includegraphics[width=0.5\textwidth]{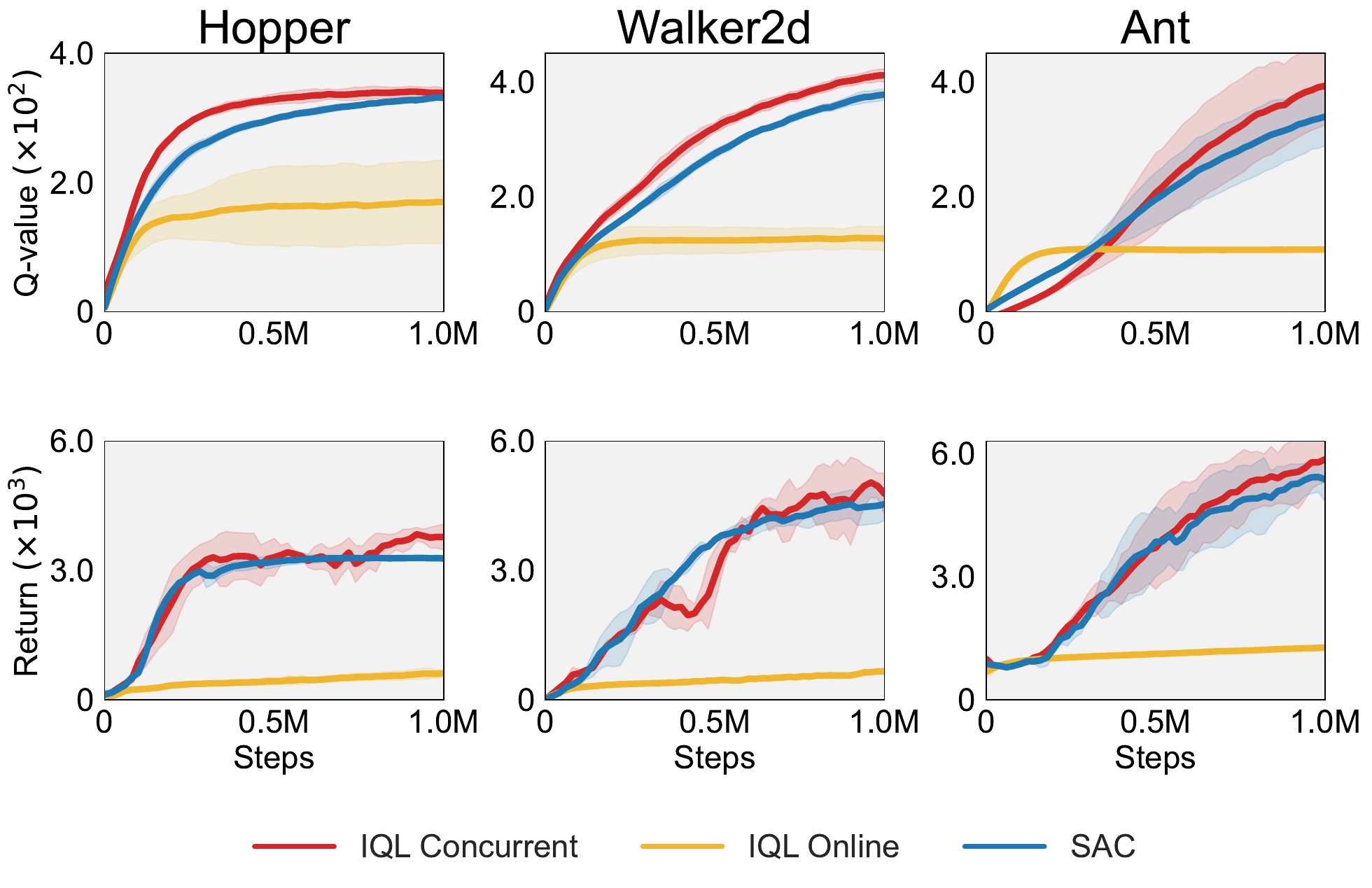}
    \caption{\textbf{Motivating example}. \textit{(Top)}: $Q$-value estimation, \textit{(Bottom)}: policy performance. We compare three agents including pure off-policy, concurrent offline and pure offline in online settings. The results demonstrate that the concurrent offline optimal policy can outperform the learning policy by sharing the replay buffer.}
    % \vspace{-15pt}
    \label{fig:motivation_example}
\end{figure}

We present the performance of each agent, alongside their $Q$-value estimations in Figure~\ref{fig:motivation_example}. In each task, when comparing the \textit{IQL concurrent} agent with the \textit{SAC} agent, we observe the potential superiority of the offline optimal policy over the online one, even though both share the same replay buffer explored by SAC. Considering the different actor training methods between IQL (by forward KL-divergence) and SAC (by reverse KL-divergence) may cause performance loss~\cite{chan2022greedification}, the Q-value comparison serves as a clearer indicator of the performance gap. These findings suggest that, contrary to its conservative reputation, offline RL can identify a potentially superior policy with a growing dataset when compared with off-policy RL. However, without the online policy buffer, even though the \textit{IQL Online} agent can collect new samples, it exhibits notable conservatism, leading to premature convergence in both $Q$-value estimation and overall performance.

Despite these interesting discoveries, the timing of the offline optimal policy's superiority is uncertain, varying across tasks. And even within a single task, it depends on the quality of online learning policy interactions. While some works~\cite{ji2023seizing,zhang2022replay} have utilized the offline optimal $Q$-value to regularize the $Q$-value of the online learning policy, the challenge arises when the offline optimal policy is inferior, potentially leading to harm to the online policy. Thus, the uncertainty in the timing of superiority makes it non-trivial to leverage the offline optimal policy effectively.

\subsection{Derivation of Offline-Boosted Policy Iteration}
To better determine the suitable timing for introducing $\mu^*_k$, we first individually evaluate both $\pi_k$ and $\mu^*_k$. Let $V^{\pi_k}(s)$ denote the state value function and $Q^{\pi_k}(s,a)$ represent the state-action value function of the online learning policy $\pi_k$. Similarly, let $V^{\mu^*_k}(s)$ and $Q^{\mu^*_k}(s,a)$ denote the value function and state-action value function for the offline optimal policy $\mu^*_k$. We exploit the Bellman Expectation Operator given by:
\begin{equation}\label{evalua_for_Q}
\mathcal{T}^\chi Q^\chi(s,a)=r(s,a)+\gamma\mathbb{E}_{s',a'\sim\chi}\left[Q^\chi(s',a')\right],
\end{equation}
\begin{equation}\label{V_evalua}
V^\chi(s)=\mathbb{E}_{a\sim\chi}\left[Q^\chi(s,a)\right], \quad \chi=\pi_k \ \text{or}\ \mu^*_k,
\end{equation}
within the replay buffer for the evaluation of $\pi_k$ and $\mu^*_k$. Since $\mathcal{T}^\chi$ is a $\gamma$-contraction mapping within a single policy evaluation step \emph{w.r.t} either $\pi_k$ or $\mu^*_k$~\cite{denardo1967contraction,bellemare2017distributional}, the (state-action) value function $Q^\chi(s,a)$ and $V^{\chi}(s)$ can be obtained by repeatedly applying $\mathcal{T}^\chi$.

Next, when performing policy improvement, we utilize the maximization of $Q^{\pi_k}(s,a)$ as the objective function, as it provides an unbiased evaluation of the current learning policy. In contrast with previous off-policy learning methods, we introduce the offline optimal policy as a guidance policy to assist in generating a new online learning policy. Specifically, using the value function $V^\chi(s)$ as the performance indicator, which measures the performance of policies at each state, we design the following adaptive mechanism for any state $s\in\mathcal{D}$:
\begin{itemize}[left=5pt]
\vspace{-8pt}
    \item When $V^{\pi_k}(s)\geq V^{\mu^*_k}(s)$, \emph{i.e.}, $\mathbb{E}_{a\sim\pi_k}[Q^{\pi_k}(s,a)]\geq\mathbb{E}_{a\sim\mu^*_k}[Q^{\mu^*_k}(s,a)]$, according to the definition of $\mu^*_k$~(\ref{offline_policy_def}), it implies that even using the optimal actions in the replay buffer, the online learning policy $\pi_k$ would still perform better than them. Thus, we can directly solve the objective function without the introduction of $\mu^*_k$, avoiding potential negative effects.
    \item When $V^{\pi_k}(s)\leq V^{\mu^*_k}(s)$, we can identify better actions in the replay buffer compared to the current learning policy. In this case, we consider adding a policy constraint to the objective function. This ensures that the updated policy not only optimizes the objective function, but also integrates the distribution of better actions from the offline optimal policy as guidance, thus leveraging the historical optimal behavior when it surpasses the online learning policy.
\vspace{-8pt}
\end{itemize}

Following this insight, we reconstruct the optimization problem in the policy improvement step as:
\begin{align}
&\quad\quad\quad\quad\quad\pi_{k+1}=\arg\max_\pi\mathbb{E}_{a\sim\pi}[Q^{\pi_k}(s,a)]\label{eq:optimization_problem}\\
&\text{s.t.}\int_{a\in\mathcal{A}}\!\!\!\!\!f\left(\frac{\pi(a|s)}{\mu^*_k(a|s)}\right)\!\!\mathbbm{1}\!\!\left(V^{\mu^*_k}(s)-V^{\pi_k}(s)\right)\mu^*_k(a|s)da\leq\epsilon,\label{eq:adative_constraint}\\
& \quad\quad\quad\quad\quad\ \int_{a\in\mathcal{A}}\pi(a|s)da=1, \ \forall s\in\mathcal{D},\label{eq:policy_norm}
\end{align}
where $f(\cdot)$ is a regularization function, and $\mathbbm{1}(\cdot)$ is an indicator function with $x\geq0,\mathbbm{1}(x)=1;x<0,\mathbbm{1}(x)=0$. Constraint~(\ref{eq:adative_constraint}) allows us to adaptively blend the offline optimal policy into online policy learning. By leveraging the Lagrangian multiplier and KKT condition~\cite{peters2010relative,peng2019advantage}, we derive the closed-form solution for the constrained optimization problem, as outlined in the following proposition.
\begin{proposition}\label{prop_closed_form}
For the constrained optimization problem defined by~(\ref{eq:optimization_problem})$\sim$(\ref{eq:policy_norm}), if $V^{\mu^*_k}(s)\geq V^{\pi_k}(s)$, the closed-form solution is
\begin{equation}\label{closed_form}
\pi_{k+1}=\frac{1}{Z(s)}\mu^*_k(a|s)\left(f'\right)^{-1}\Big(Q^{\pi_k}(s,a)\Big),
\end{equation}
where $Z(s)$ is a partition function to normalise the action distribution. Or, when $V^{\mu^*_k}(s)<V^{\pi_k}(s)$, $\pi_{k+1}$ is an ordinary solution to maximize $Q^{\pi_k}(s,a)$.
\end{proposition}

The proof is provided in Appendix~\ref{sec:Theoretical_Analyses}. Then, based on this closed-form solution, we show that the newly generated learning policy $\pi_{k+1}$ would have a higher value than the old learning policy $\pi_k$ \emph{w.r.t.} the state-action distribution of the replay buffer in the following proposition.
\begin{proposition}
Let $\pi_k$ be the older learning policy and the newer one $\pi_{k+1}$ be the solution of~(\ref{eq:optimization_problem})$\sim$(\ref{eq:policy_norm}). Then we achieve $Q^{\pi_{k+1}}(s,a)\geq Q^{\pi_k}(s,a)$ for all $(s,a)\in\mathcal{D}$, with the offline optimal policy $\mu^*_k$ serving as a performance baseline policy.
\end{proposition}

With the convergence of policy evaluations on $\pi_k$ and $\mu^*_k$, as well as the results of policy improvement, we can alternate both steps and the online learning policy would provably converge to the optimal policy.

\begin{proposition}
Assume $\vert \mathcal{A}\vert < \infty$, repeating the alternation of the policy evaluation~(\ref{evalua_for_Q})$\sim$(\ref{V_evalua}) and policy improvement~(\ref{eq:optimization_problem})$\sim$(\ref{eq:policy_norm}) can make any online learning policy $\pi_k\in\Pi$ converge to the optimal policies $\pi^*$, s.t. $Q^{\pi^*}(s_t,a_t)\geq Q^{\pi_k}(s_t,a_t), \forall (s_t,a_t) \in \mathcal{S}\times \mathcal{A}$.
\end{proposition}

\subsection{Offline-Boosted Actor Critic}
Inspired by the previous analyses, we introduce our method for online off-policy RL, Offline-Boosted Actor-Critic (OBAC), summarized in Algorithm~\ref{alg:OBAC}. To extend OBAC to large continuous control domains, we derive a practical implementation with high-quality function approximators following previous works~\cite{haarnoja2018soft,lee2020stochastic}. Besides, we highlight that OBAC introduces implicit regularization of the offline optimal policy for online policy training, without explicitly learning it, thereby mitigating computational complexity at each iteration.

Specifically, we consider parameterized state value functions $V^\chi_{\psi}(s)$ and state-action value functions $Q^\chi_\phi(s,a)$ with parameters $\psi$ and $\phi$, where $\chi$ represents both the online learning policy $\pi$ and the offline optimal policy $\mu^*$, alongside a tractable online learning policy $\pi_\theta(a|s)$ defined as a Gaussian policy with parameters $\theta$. We utilize the updated online learning policy to interact with the environment, collecting new samples $\{(s,a,s',r)\}$ into the buffer $\mathcal{D}$.

\noindent\textbf{Policy Evaluation.}\quad
In this step, we first derive $V^\pi(s)$ and $Q^\pi_\phi(s,a)$ through the Bellman Expectation Operator~(\ref{evalua_for_Q}), by minimizing the squared residual error
\begin{align}\label{evalua_Q_pi_k}
\arg\min_{Q^\pi_\phi}\!\mathbb{E}_{(s,a,r,s')\sim\mathcal{D}}\!\left[\frac{1}{2}\left(Q^\pi_\phi(s,a)\!-\!\mathcal{T}^\pi Q^\pi_\phi(s,a)\right)^2\right]
\end{align}
and then we compute the value function $V^\pi(s)$ forward without gradient step, 
\begin{equation}\label{evalua_V_pi_k}
V^\pi(s)=\mathbb{E}_{a\sim\pi}[Q^\pi_\phi(s,a)].
\end{equation}

Recalling the definition~(\ref{offline_policy_def}) of the offline optimal policy $\mu^*_k=\arg\max_{a\sim\mathcal{D}}Q^{\mu_k}(s,a)$. To eliminate the requirement of $\mu^*_k$ in the evaluation step, we transfer the Bellman Expectation Operator $\mathcal{T}^{\mu^*_k}$ as
\begin{align}\label{max_mu_evaluation}
&\mathcal{T}^{\mu^*_k}Q^{\mu^*_k}_\phi(s,a)=r(s,a)+\gamma\mathbb{E}_{s',a'\sim\mu^*_k}\left[Q^{\mu^*_k}_\phi(s',a')\right]\nonumber\\
&\quad\quad=r(s,a)+\gamma\mathbb{E}_{s'}\left[\max_{a'\sim\mathcal{D}}Q^{\mu^*_k}_\phi(s',a')\right].
\end{align}

Within Equation~(\ref{max_mu_evaluation}), prior works on offline RL have effectively addressed $\max_{a'\sim\mathcal{D}}$ without explicitly requiring $\mu^*_k$, such as expectile regression used in IQL~\cite{kostrikov2021offline}. For simplicity, we use expectile regression to achieve $Q^{\mu^*}(s,a)$ and $V^{\mu^*_k}(s)$, with two specific steps:
\begin{equation}\label{evalua_V_mu}
\arg\min_{V^{\mu^*}_\psi}\mathbb{E}_{(s,a)\sim\mathcal{D}}\left[L^\tau_2\left(Q^{\mu^*}_\phi(s,a)-V^{\mu^*}_\psi(s)\right)\right]
\end{equation}
where $L^\tau_2(x)=\vert\tau-\mathbbm{1}(x<0)\vert x^2$ is the expectile regression function, and $\tau$ is an expectile factor. And,
\begin{equation}\label{evalua_Q_mu}
\arg\min_{Q^{\mu^*}_\phi}\mathbb{E}_{(s,a,s',r)\sim\mathcal{D}}\left[\frac{1}{2}\left(r+\gamma V^{\mu^*}_\psi(s')-Q^{\mu^*}_\phi(s,a)\right)^2\right].
\end{equation}

Based on Theorem 3 in ~\citet{kostrikov2021offline}, when $\tau\rightarrow 1$, the term $\max_{a'\sim\mathcal{D}}Q^{\mu^*_k}_\phi(s',a')$ can be approached
to derive $Q^{\mu^*_k}_\phi(s,a)$. Thus, we complete the policy evaluation for both $\pi_k$ and $\mu^*_k$, where the former is based on~(\ref{evalua_Q_pi_k}) and~(\ref{evalua_V_pi_k}), and the latter by~(\ref{evalua_V_mu}) and~(\ref{evalua_Q_mu}). In our implementation, we employ the Clipped Double Q-technique~\cite{fujimoto2018addressing} for stability and mitigating overestimation.

\begin{algorithm}[t]
% \vspace{-5pt}
   \caption{Offline-Boosted Actor-Critic (OBAC)}
   \label{alg:OBAC}
\begin{algorithmic}[1]
   \STATE {\bfseries Input:} Critic $Q^\pi_\phi$, critic $Q^{\mu^*}_\phi$, value $V^{\mu^*}_\psi$, actor $\pi_\theta$, replay buffer $\mathcal{D}$.
   \REPEAT
       \FOR{each environment step}
           \STATE $a\sim\pi_\phi(a|s)$ and $r,s'\sim\mathcal{P}(s'|s,a)$
           \STATE $\mathcal{D}\leftarrow\mathcal{D}\cup\left\{(s,a,s',r)\right\}$
       \ENDFOR
       \FOR{each gradient step}
           \STATE \textit{For} $\pi$: Update $Q^\pi_\phi$ by~(\ref{evalua_Q_pi_k}), compute $V^\pi$ by~(\ref{evalua_V_pi_k})
           \STATE \textit{For} $\mu^*$: Update $Q^{\mu^*}_\phi$ by~(\ref{evalua_Q_mu}), update $V^{\mu^*}_\psi$ by~(\ref{evalua_V_mu})
           \STATE Update $\pi_\theta$ by~(\ref{eq:update_pi})
       \ENDFOR
   \UNTIL{the policy performs well in the environment}
\end{algorithmic}
\end{algorithm}

\noindent\textbf{Policy Improvement.}\quad
Directly using the closed-form solution~(\ref{closed_form}) for policy updates is intractable with the unknown $Z(s)$ and $\mu^*_k$. In our implementation, we opt to restrict the solution within a tractable set of Gaussian policies and project the improved policy into these desired policies by Kullback-Leibler divergence $\KL(\cdot)$. Then, if we choose the regularization function $f(x)=x\log x$, the objective of the updated policy is
\begin{equation}\label{eq6}
\arg\min_{\pi}\left\{
\begin{aligned}
&\KL\left(\pi\Big\Vert\frac{\exp(Q^{\pi_k})}{Z}\right), V^{\pi_k}\geq V^{\mu^*_k}, \\
&\KL\left(\pi\Big\Vert\frac{\mu^*_k\exp( Q^{\pi_k})}{Z}\right), V^{\pi_k}\leq V^{\mu^*_k}.
\end{aligned}
\right.
\end{equation}

\begin{figure*}[t]
    \centering
    \includegraphics[width=0.24\linewidth]{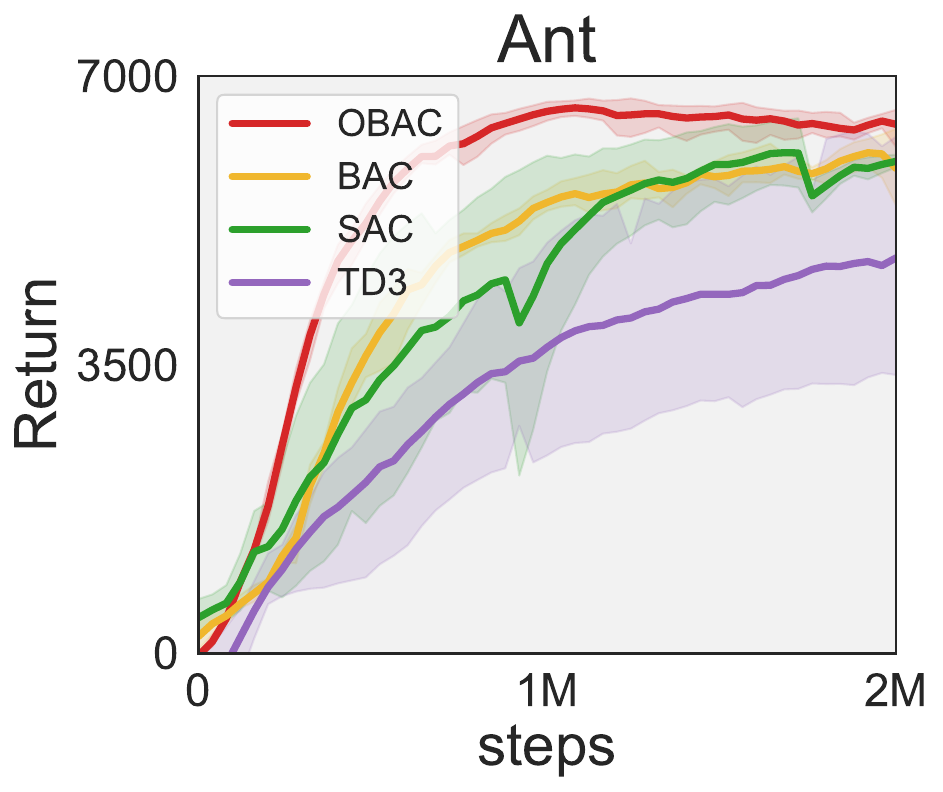}
    \includegraphics[width=0.25\linewidth]{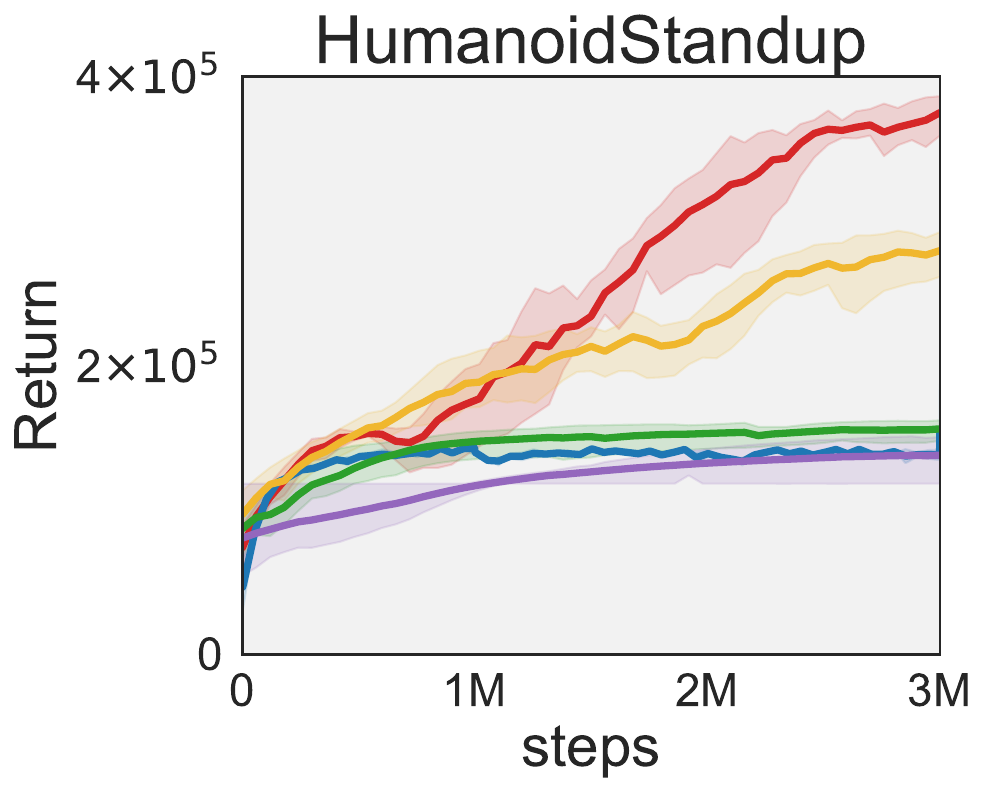}
    \includegraphics[width=0.24\linewidth]{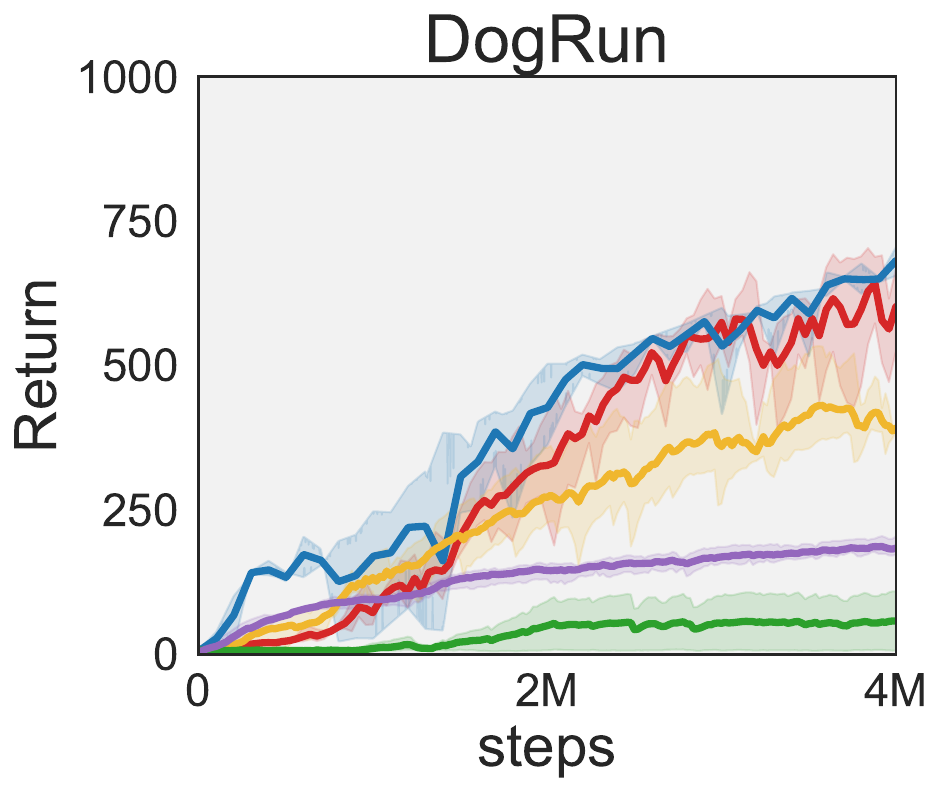}
    \includegraphics[width=0.24\linewidth]{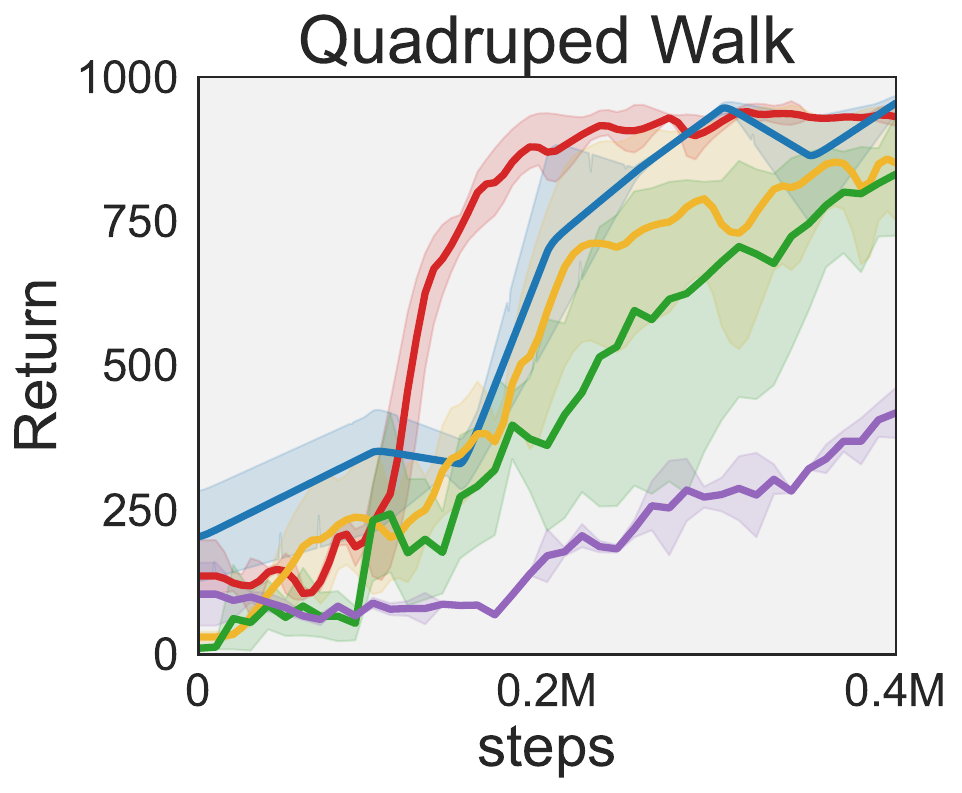}
    \includegraphics[width=0.24\linewidth]{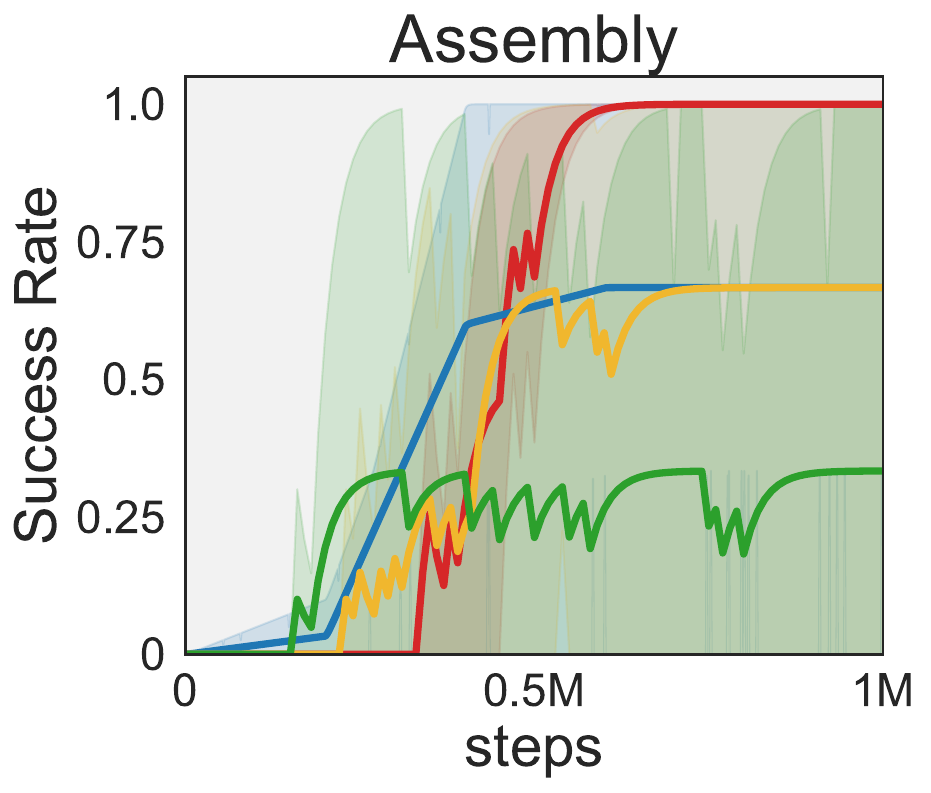}
    \includegraphics[width=0.24\linewidth]{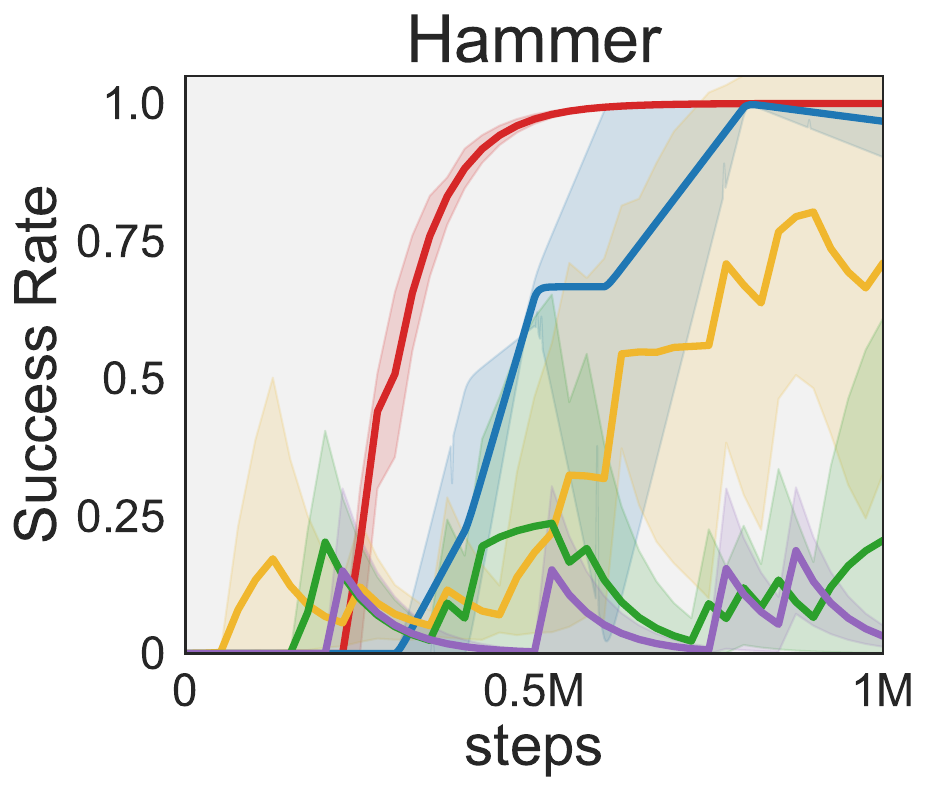}
    \includegraphics[width=0.24\linewidth]{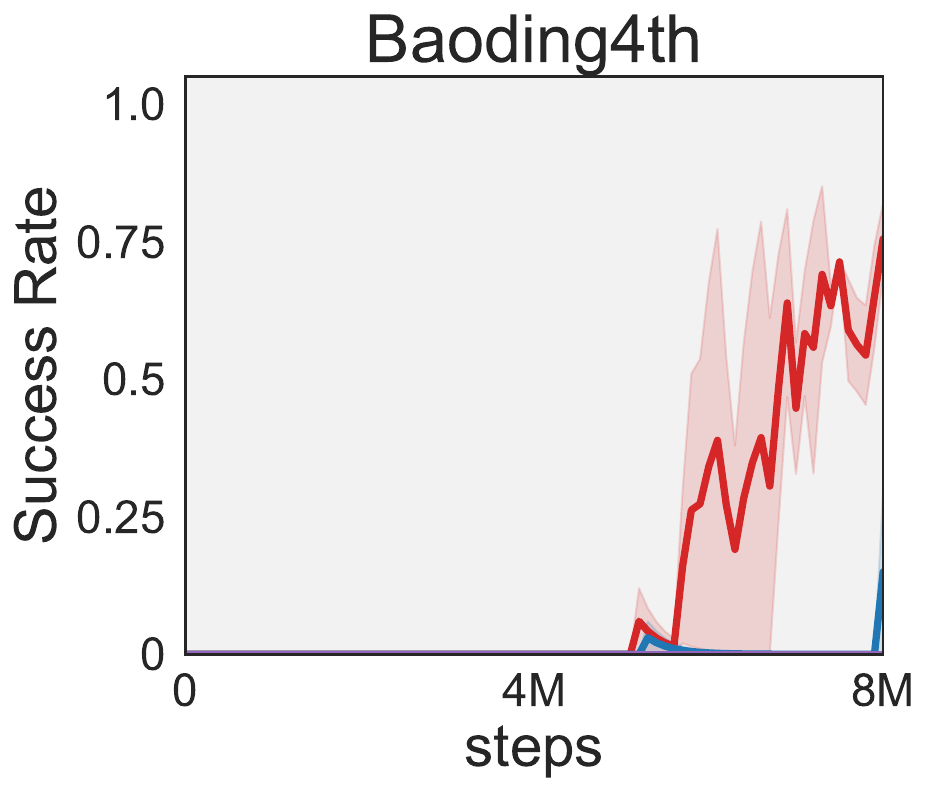}
    \includegraphics[width=0.24\linewidth]{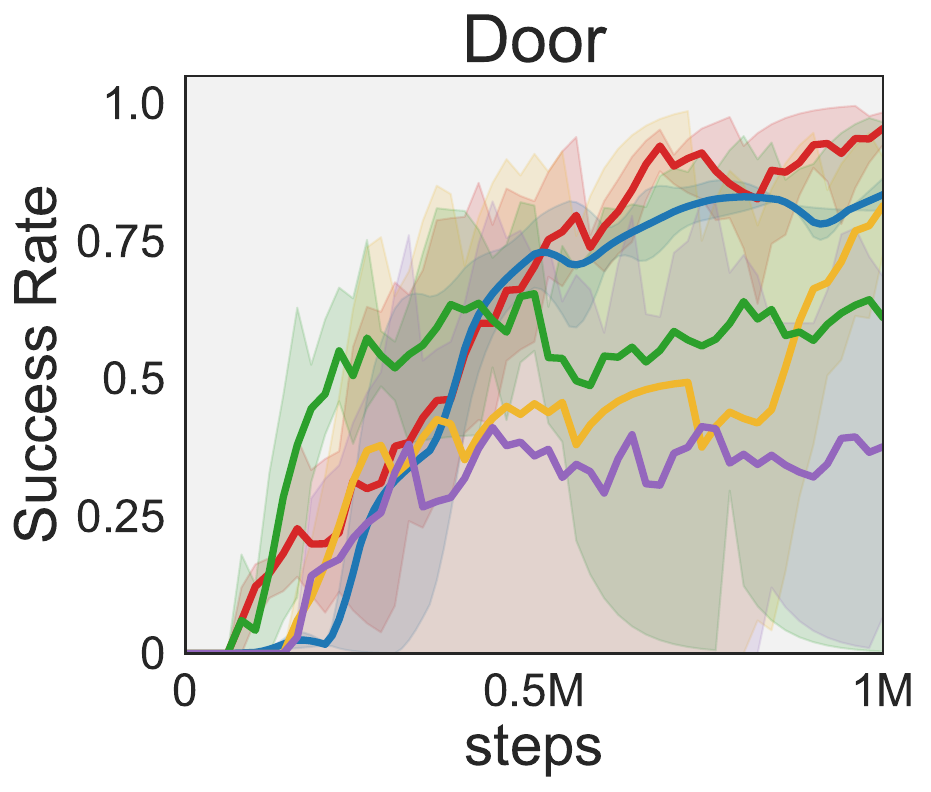}
    \includegraphics[width=0.24\linewidth]{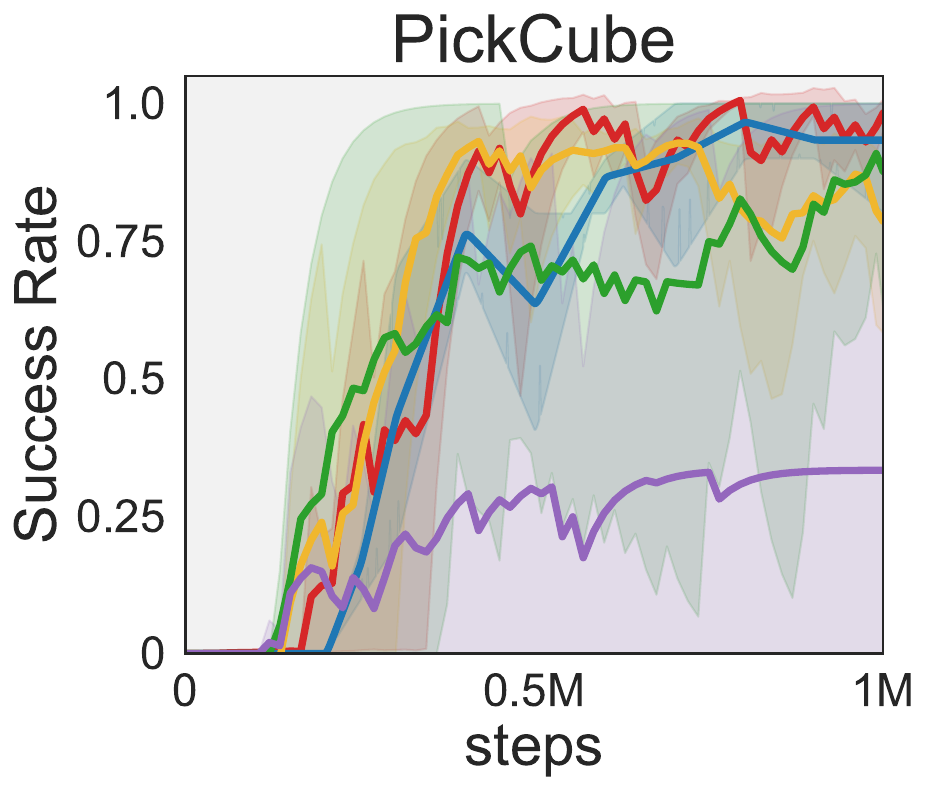}
    \includegraphics[width=0.25\linewidth]{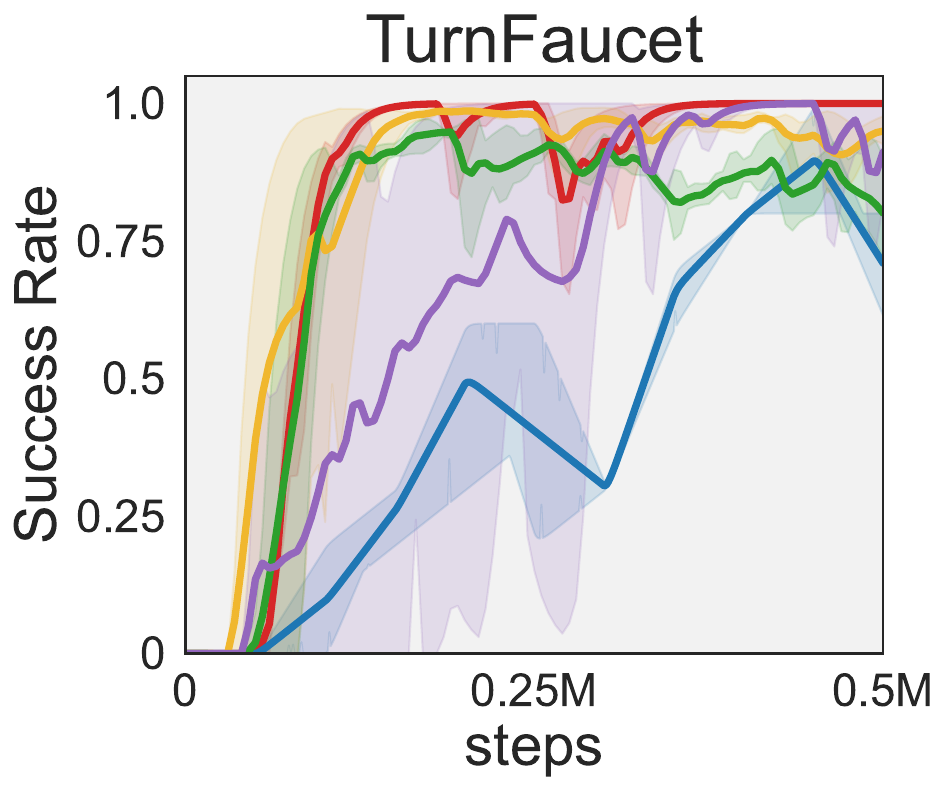}
    \includegraphics[width=0.24\linewidth]{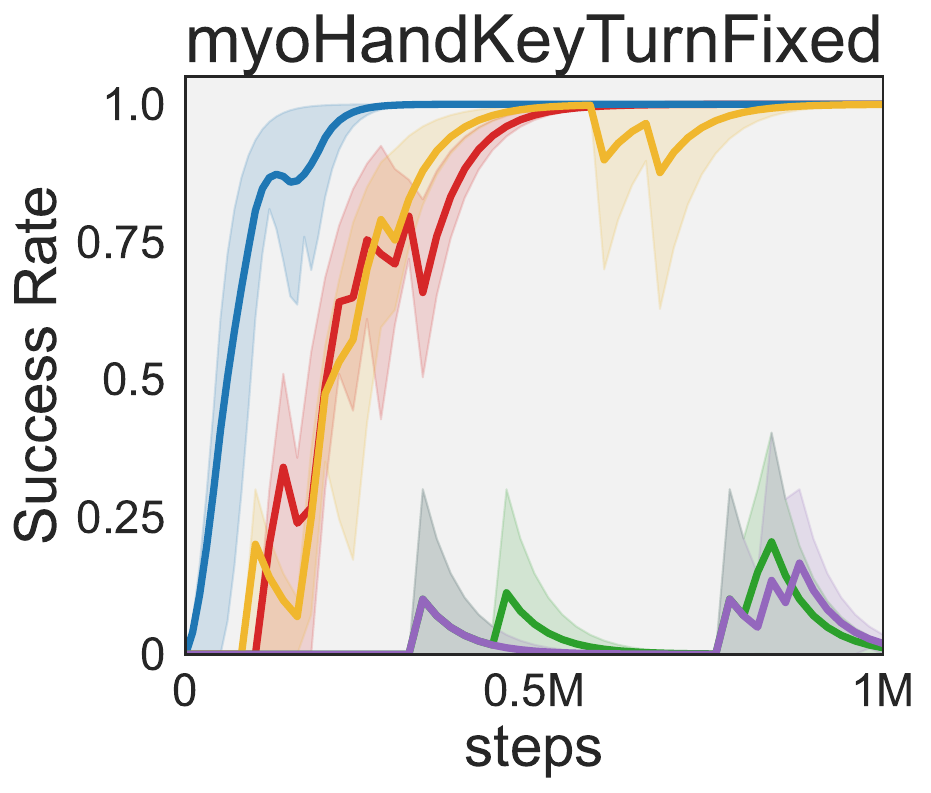}
    \includegraphics[width=0.24\linewidth]{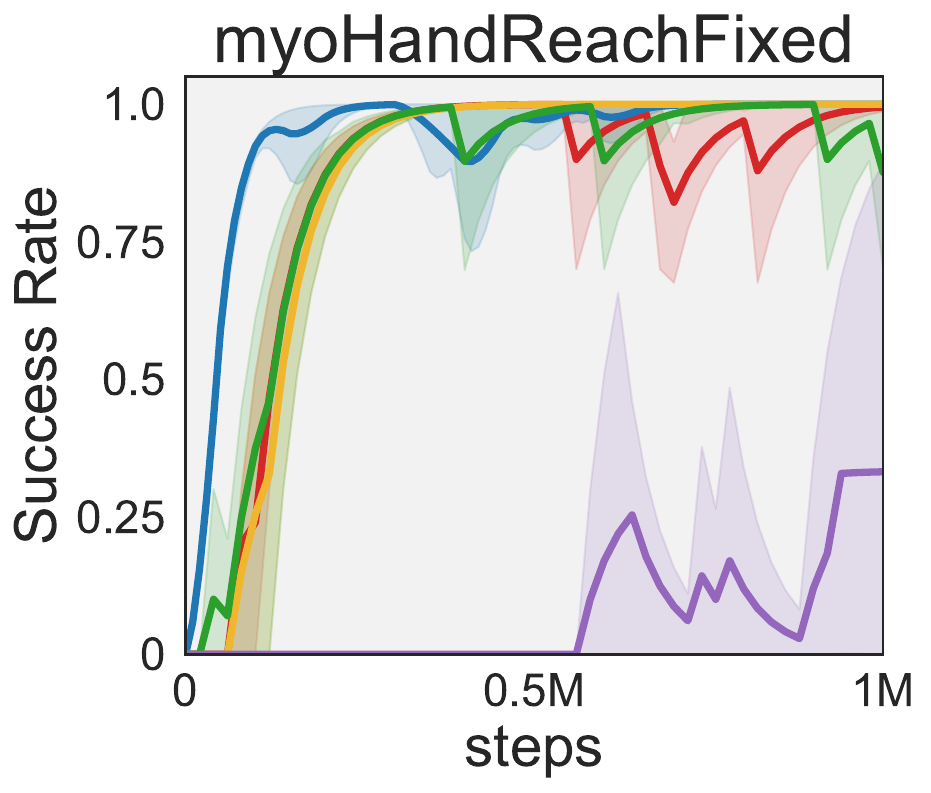}
    \includegraphics[width=0.5\linewidth]{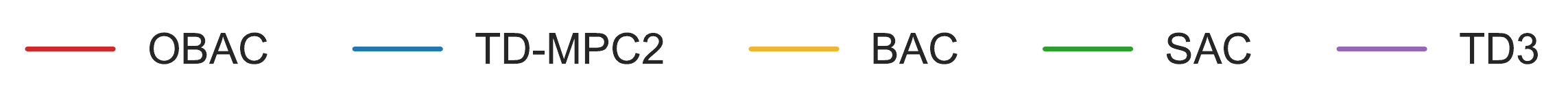}
    \caption{\textbf{Main results.} We provide performance comparisons for $12$ of the $53$ tasks, two for each task suite. Please refer to Appendix~\ref{sec:all_results} for the comprehensive results. The solid lines are the average return/success rate, while the shades indicate $95\%$ confidence intervals. All algorithms are evaluated with $5$ random seeds.}
    \label{fig:Main_results}
\end{figure*}
Given the data $\{(s,a,s',r)\}$ randomly sampled from $\mathcal{D}$, if $V^{\mu^*_k}(s)<V^{\pi_k}(s)$, we disable the constraint and update the policy as prior off-policy RL methods. In contrast, if $V^{\mu^*_k}(s)\geq V^{\pi_k}(s)$, we consider the distribution of sampled data can achieve better performance than the learning policy, thus it can approximate the offline optimal policy~\cite{zhang2022replay}. Thus, with the Gaussian policy set, we summarise both cases for policy updating:
\begin{align}\label{eq:update_pi}
&\arg\min_{\pi_\theta\in\Pi}\mathbb{E}_{s\sim\mathcal{D}}\Big\{\mathbb{E}_{a\sim\pi_\theta}[\log\pi_\theta(a|s)-Q^{\pi_k}(s,a)]\nonumber\\
&\quad-\lambda\mathbb{E}_{a\sim\mathcal{D}}[\log\pi_\theta(a|s)]\mathbbm{1}\left(V^{\mu^*_k}(s)-V^{\pi_k}(s)\right)\Big\}.
\end{align}
where $\lambda$ is a behavior clone weighted factor, similar to previous offline RL works~\cite{kostrikov2021offline}.

\section{Experiment}
We evaluate OBAC across \textbf{53} continuous control tasks spanning \textbf{6} domains: Mujoco~\cite{todorov2012mujoco}, DMControl~\cite{tassa2018deepmind}, Meta-World~\cite{yu2020meta}, Adroit~\cite{kumar2015mujoco}, Myosuite~\cite{caggiano2022myosuite}, and Maniskill2~\cite{gu2022maniskill2}. These tasks cover a wide range of challenges, including high-dimensional states and actions (up to $\mathcal{S}\in\mathbb{R}^{375}$ and $\mathcal{A}\in\mathbb{R}^{39}$), sparse rewards, multi-object and delicate manipulation, musculoskeletal control, and complex locomotion. Please refer to Appendix~\ref{sec:implementation} for the implementation details and environment settings in our experiments.

With these experimental evaluations, we seek to investigate the following questions: 1) Does the introduction of the concurrent offline optimal policy significantly improve performance? 2) How does OBAC compare to the popular model-free and model-based RL methods for sample efficiency and eventual performance? 3) How does the adaptive mechanism work in OBAC?

\noindent\textbf{Baselines.}\quad 
Our baselines contain: 1) \textbf{SAC}~\cite{haarnoja2018soft} and \textbf{TD3}~\cite{fujimoto2018addressing}, two data-efficient off-policy model-free RL methods, where the former utilizes stochastic policy while the latter uses deterministic policy; 2) \textbf{BAC}~\cite{ji2023seizing}, an off-policy model-free RL method to employ the state-action value of offline optimal policy to mitigate the underestimation of the $Q$-value of online learning policy; 3) \textbf{TD-MPC2}~\cite{hansen2023td}, a high-efficient model-based RL method that combines model predictive control and TD-learning.

\subsection{Experimental Results}
\noindent\textbf{Performance comparison.}\quad
The learning curves presented in Figure~\ref{fig:Main_results} demonstrate the performance of OBAC alongside various baselines across diverse task suites. Overall, OBAC, as a model-free RL method, obviously outperforms other model-free RL baselines and exhibits comparable capabilities to model-based RL methods in terms of exploration efficiency and asymptotic performance. Notably, with identical hyperparameters, OBAC excels in high-dimensional locomotion tasks (DogRun), multi-body contact manipulation (Baoding4th), and intricate muscle control tasks (myoHand). A noteworthy achievement is the successful application of OBAC to solve the Baoding-4th task, a challenging scenario involving a shadow hand managing the rotation cycle of two Baoding balls, without any prior demonstrations. Additionally, we observe that TD-MPC2, even with different prediction horizons (ranging from 3 to 6), does not perform well on the Mujoco suite, possibly influenced by variations of the \textit{done} signal settings of the environments (please refer to Appendix~\ref{sec:all_results} for more discussion). We provide a comprehensive presentation of experimental
results in Appendix~\ref{sec:all_results}, with accompanying visualizations provided in \textcolor{blue}{\url{https://roythuly.github.io/OBAC/}}.

\begin{figure}[t] 
    \centering
    \includegraphics[width=0.9\linewidth]{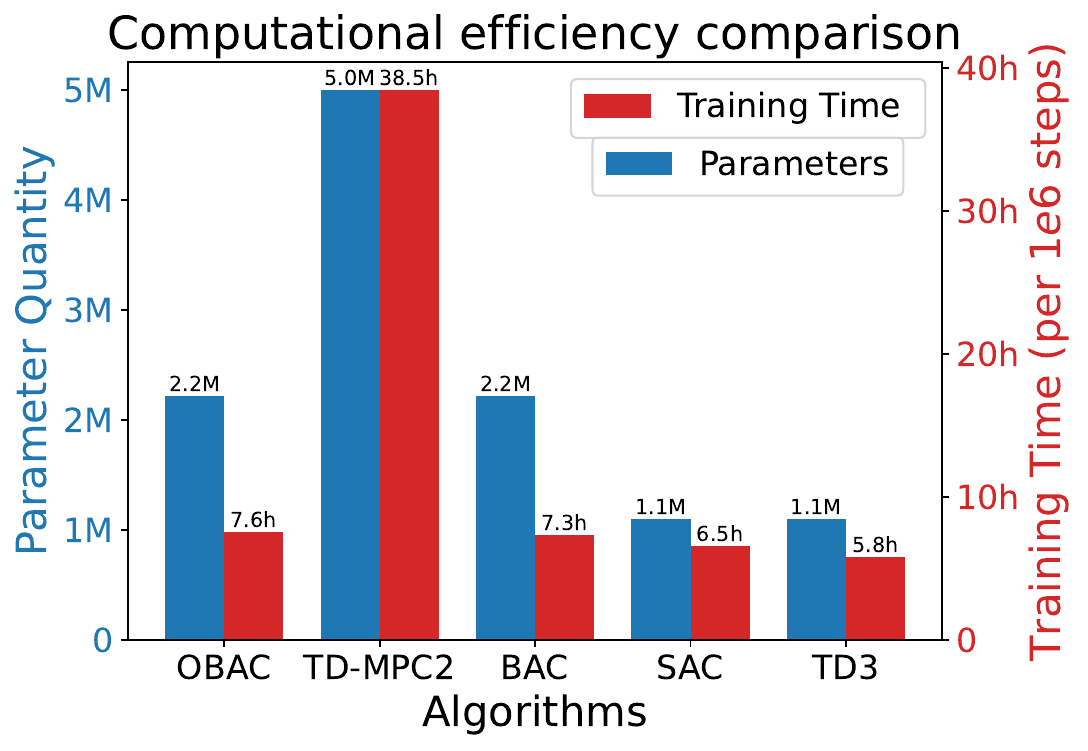}
    \caption{\textbf{Parameters and Wall time.} \textit{(left)}: parameter quantity of each algorithm, and \textit{(right)}: training wall time per 1 million steps of environmental interaction. With these comparisons, OBAC can achieve similar simplicity with other model-free RL methods, and require only $\mathbf{50\%}$ parameters and $\mathbf{20\%}$ training time compared to model-based RL methods.}
    \vspace{-10pt}
    \label{fig:Computational_performance}
\end{figure}
\noindent\textbf{Computational efficiency.}\quad
To showcase OBAC's efficiency, we provide a detailed comparison with our baselines in Figure~\ref{fig:Computational_performance}, considering parameter quantity and training wall time on a single RTX3090 GPU. In contrast to model-based RL methods TD-MPC2, OBAC achieves a $\mathbf{50\%}$ reduction in parameter quantity, featuring a simpler architecture with just two critics, one value and one actor, each component employing a 3-layer MLP with a hidden-layer size of 512. In comparison, TD-MPC2 requires additional components such as a dynamics model, reward model, and 5 ensemble critics. Furthermore, OBAC exhibits notable improvements in training efficiency, requiring only $\mathbf{20\%}$ of the training time per 1 million steps of environmental interaction while maintaining comparable sample efficiency and convergent performance. When compared with other model-free RL methods, OBAC demonstrates a similar level of simplicity and cost-effectiveness, highlighting the effectiveness of our algorithmic implementation.

\subsection{Ablation Studies}
We conduct several investigations to ablate the effectiveness of OBAC's design in this section. 

\noindent\textbf{Necessity of adaptive constraints.}\quad
One of the key designs of OBAC is the adaptive policy constraint, which dynamically adjusts based on value comparisons. To evaluate its necessity, we conducted experiments comparing \textit{OBAC (Adaptive)} with a fixed constraint \textit{OBAC (Fixed)}, where policy updates are consistently constrained by the empirical distribution of the replay buffer, as well as \textit{OBAC (Without)} constraint. The top of Figure~\ref{fig:walker2d-fix-ablation} illustrates that the fixed constraint leads to performance degradation due to excessive conservatism, underscoring the effectiveness of our adaptive mechanism for performance improvement. Additionally, the comparison of $Q$ values between the offline optimal policy and the learning policy is visualized at the bottom of Figure~\ref{fig:walker2d-fix-ablation}. The blue region indicates $Q^{\mu^*_k}\leq Q^{\pi_k}$ while the red region signifies $Q^{\mu^*_k}\geq Q^{\pi_k}$. The combined results indicate that when $Q^{\mu^*_k}\geq Q^{\pi_k}$ occurs, activating the policy constraint significantly improves OBAC's performance.

\begin{figure}[t]
    \centering
    \vspace{-6pt}
    \includegraphics[width=0.49\linewidth]{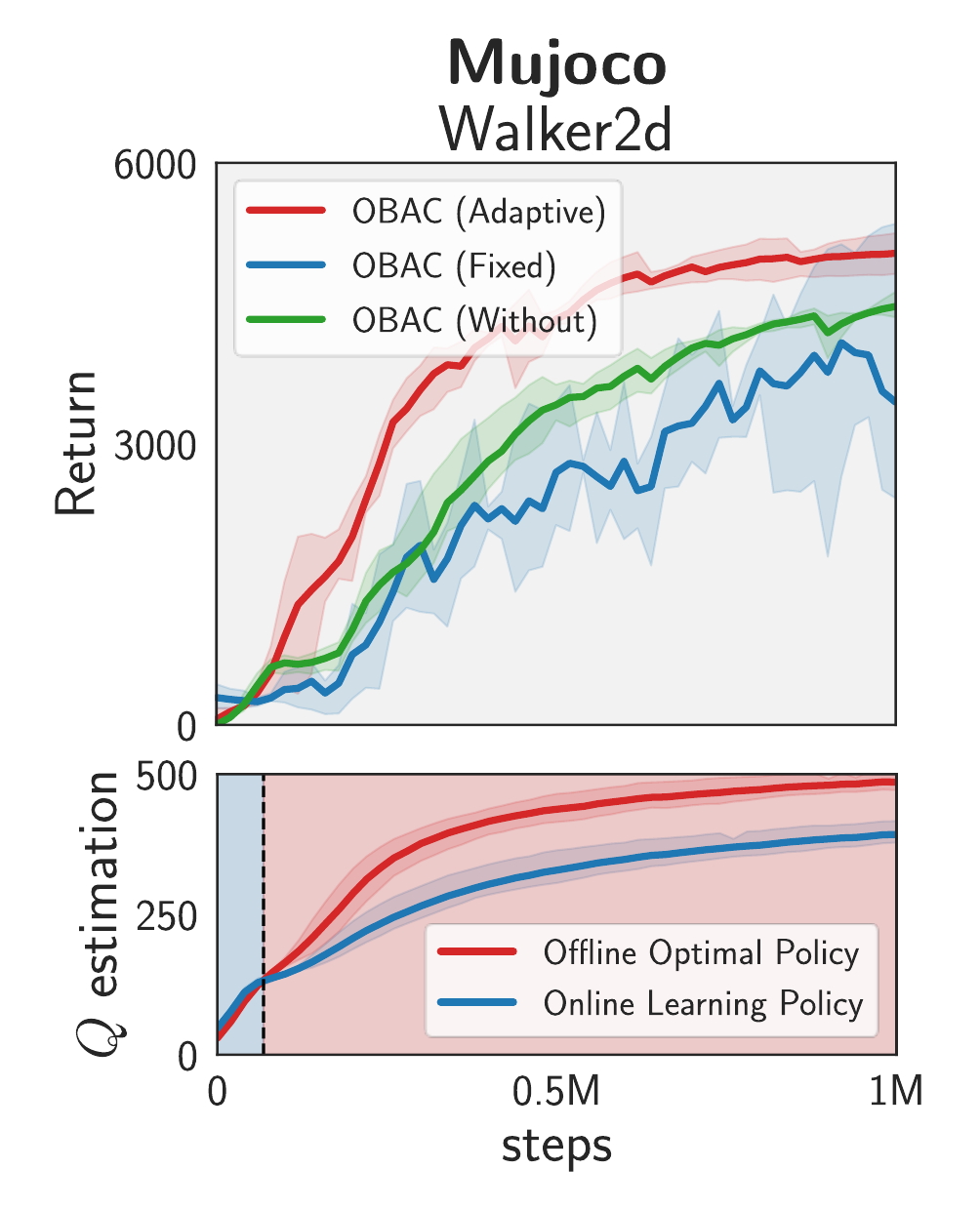}
    \includegraphics[width=0.49\linewidth]{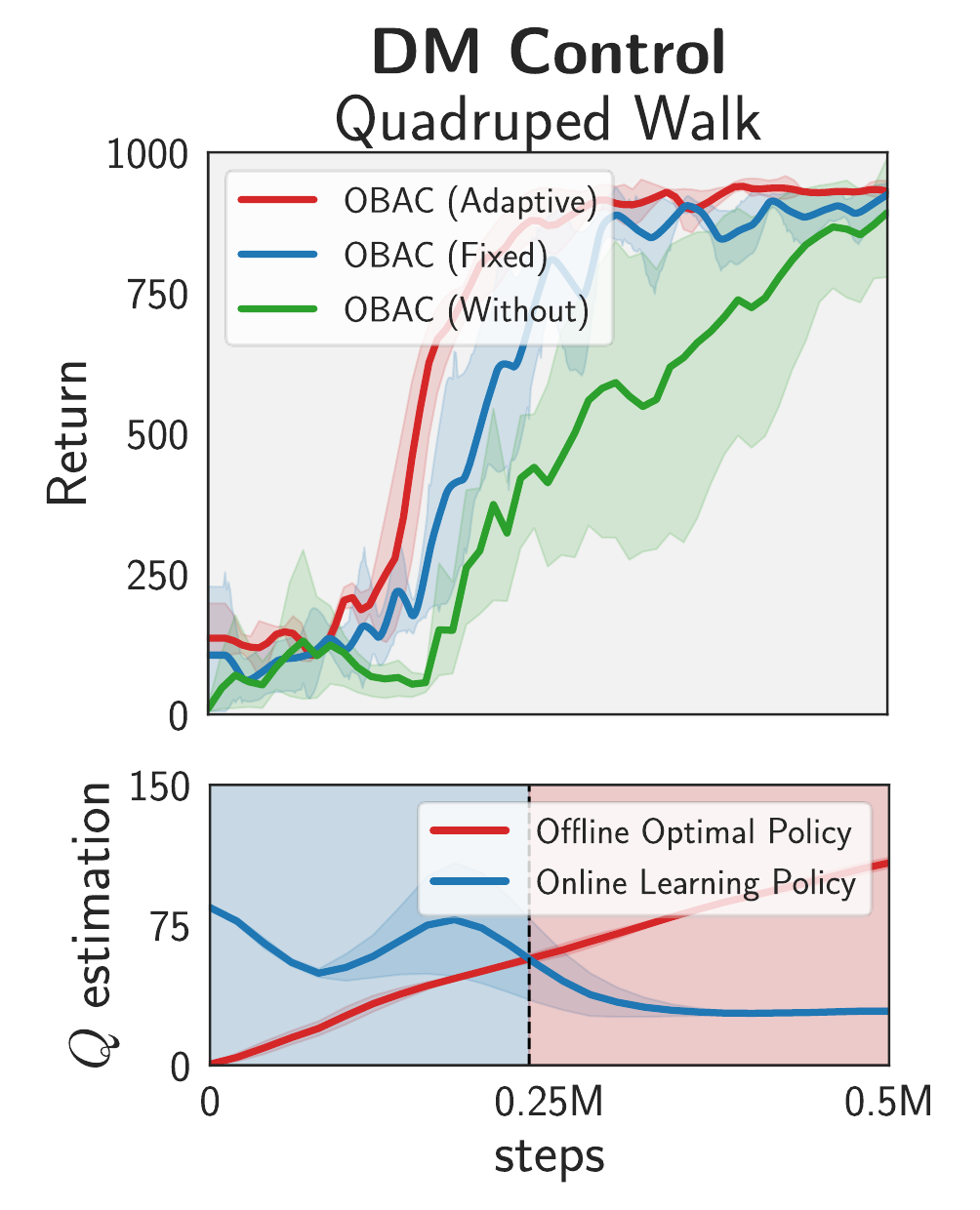}
    \caption{\textbf{Ablation on adaptive constraints.} \textit{(Top)}: Policy performance, and \textit{(Bottom)}: $Q$ value comparison. We adopt two tasks from Mujoco suite and DMControl suite to demonstrate the necessity of adaptive constraints in OBAC. Mean of 5 runs; shaded areas are $95\%$ confidence intervals.}
    % \vspace{-13pt}
    \label{fig:walker2d-fix-ablation}
\end{figure}
\begin{figure*}[t]
    \centering
    \begin{subfigure}{0.49\linewidth}
        \centering
        \includegraphics[width=0.49\linewidth]{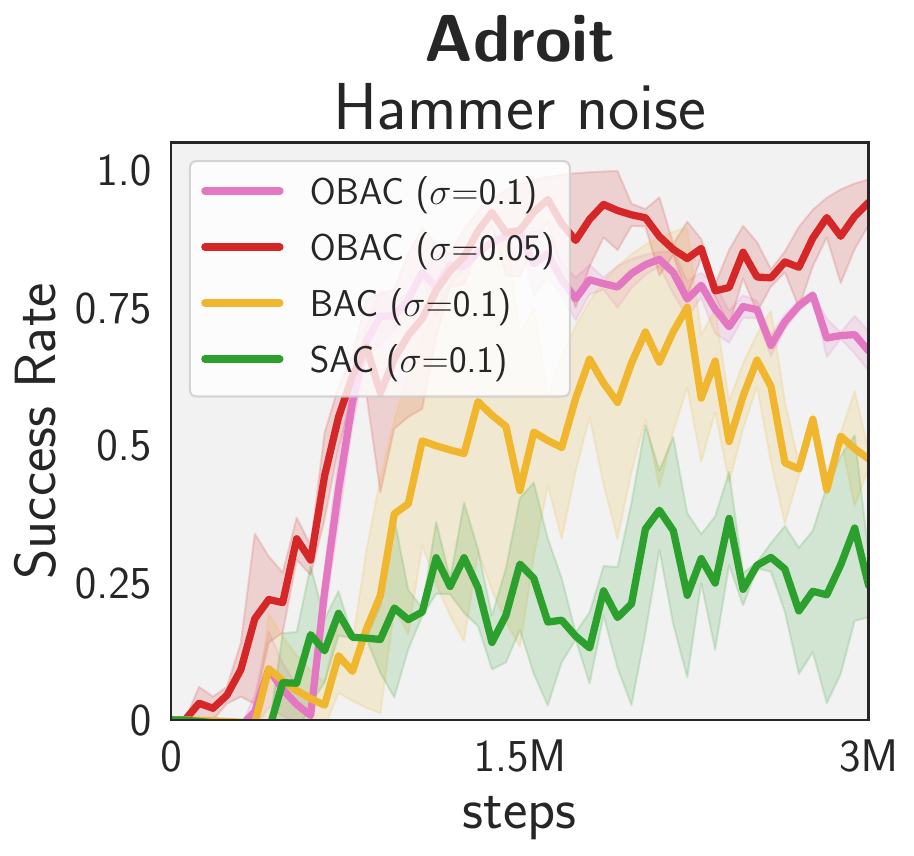}
        \includegraphics[width=0.49\linewidth]{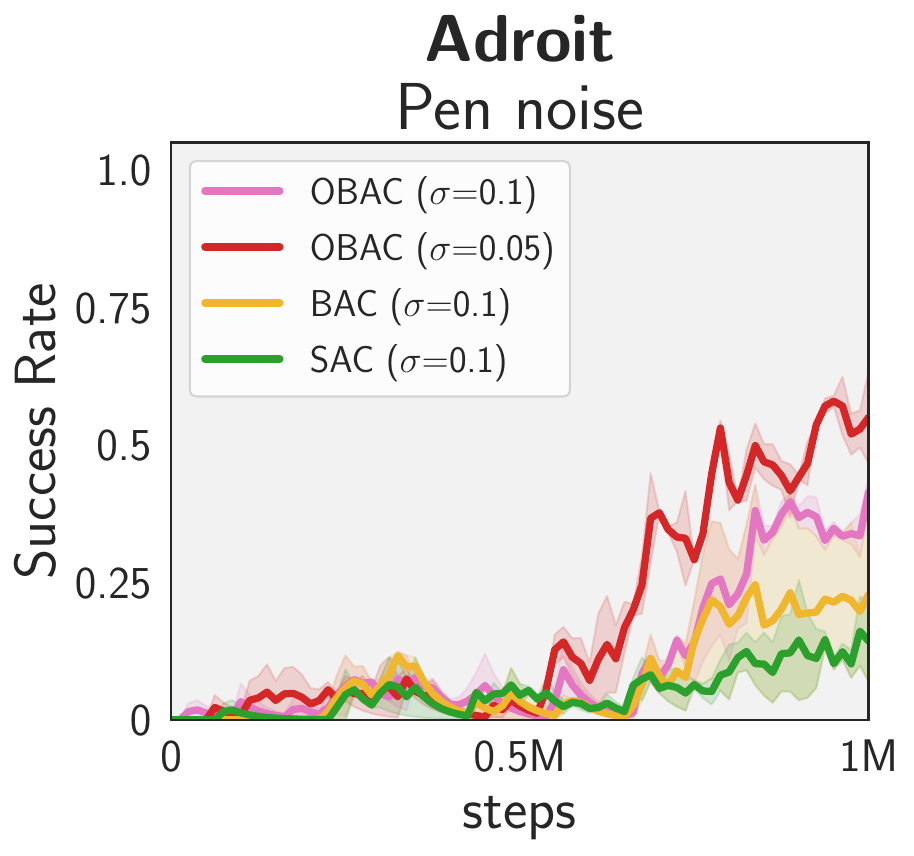}
        \caption{\textbf{Noise tasks.}}
        \label{fig:noise_ablation}
    \end{subfigure}
    \hfill
    \begin{subfigure}{0.49\linewidth}
        \centering
        \includegraphics[width=0.49\linewidth]{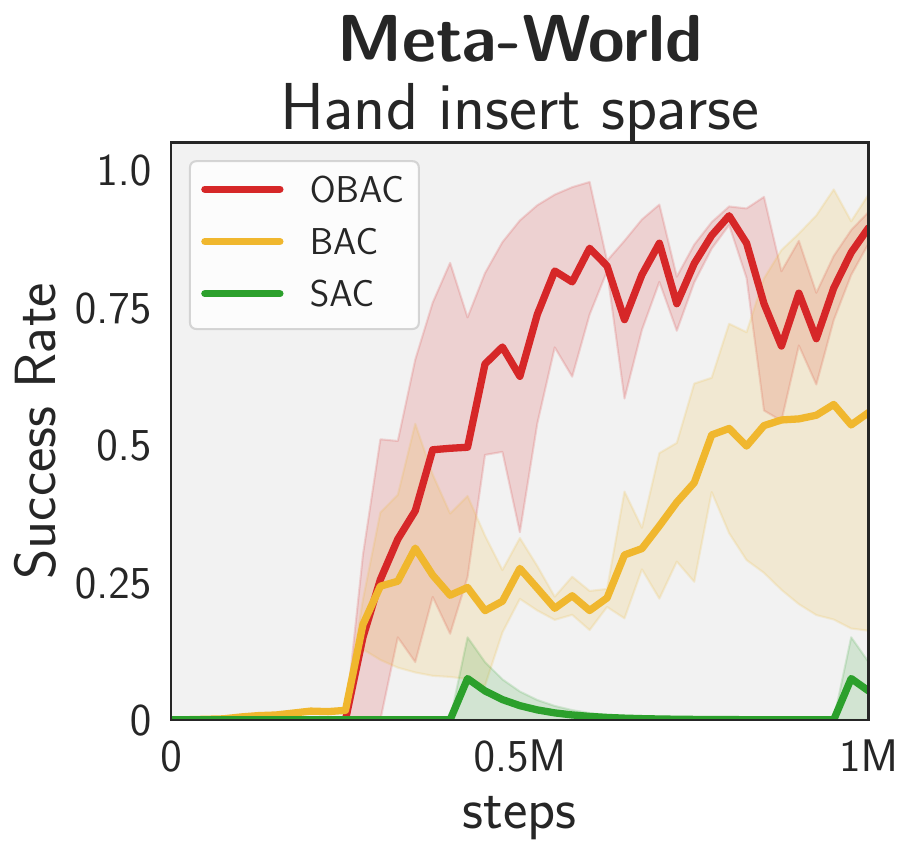}
        \includegraphics[width=0.49\linewidth]{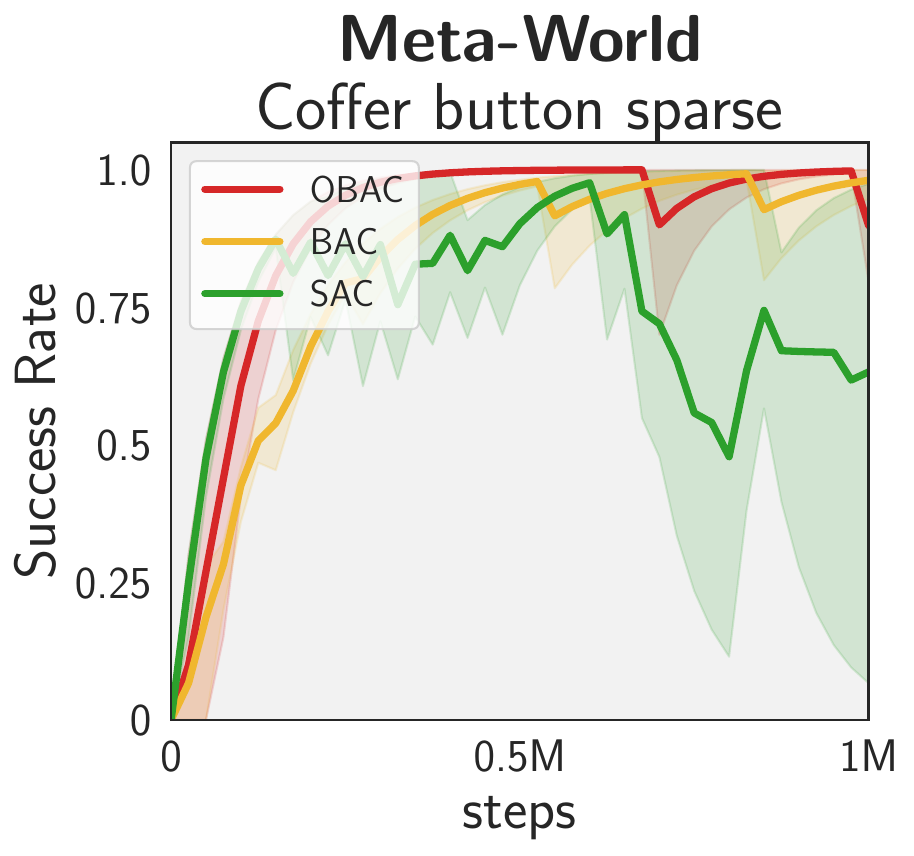} 
        \caption{\textbf{Sparse tasks.}}
        \label{fig:sparse_ablation}
    \end{subfigure}
    \caption{\textbf{Extension on noise and sparse tasks.} We evaluate OBAC in noisy and sparse tasks for comprehensive investigations. Mean of $5$ runs, shaded areas are $95\%$ confidence intervals.}
    \label{fig:ablation_figure}
\end{figure*}

\noindent\textbf{Extension in noise and sparse tasks.}\quad
To better ground the effectiveness of OBAC, we applied it to a set of stochastic tasks, where a Gaussian noise with different standard deviations $\sigma$ to each dimension of actions at every step, including $\sigma=0.05$ and $\sigma=0.1$ to represent the increasing environmental stochasticity. Compared to our baselines BAC and SAC, as shown in Figure~\ref{fig:noise_ablation}, OBAC consistently demonstrates superior robustness, even in challenging scenarios with higher noise levels (\emph{e.g.}, $\sigma=0.1$). Furthermore, we extended the evaluation to tasks featuring sparse rewards, where success yields a reward of $1$ and failure results in a reward of $0$. The results in Figure~\ref{fig:sparse_ablation} highlight OBAC's capacity to outperform the baselines in sparse reward scenarios, underscoring its effectiveness across diverse challenges. Appendix~\ref{appendix_noise} provides a detailed comparison of performance decline rates to better explain the robustness of OBAC.

\noindent\textbf{More Ablations.}\quad
In addition to these results, we provide the variant of OBAC, where a deterministic policy is employed for online learning, as shown in Figure~\ref{fig:variant_results}, Appendix~\ref{sec:more_ablation}. The results show that when employing a deterministic online learning policy, the adaptive constraint can also provide performance improvement across various scenarios, indicating the versatility of our methods.

\section{Related Works}
We briefly summarize relevant methods about off-policy RL, offline RL and off-policy RL with offline techniques.

\noindent\textbf{Off-policy RL.}\quad
Off-policy RL offers a general paradigm for reusing previously collected data in current policy training~\cite{munos2016safe,peng2019advantage,rakelly2019efficient,duan2020minimax}. This involves alternating between policy evaluation, where the performance is assessed by computing a value function, and policy improvement, where the policy is updated based on the value function~\cite{chan2022greedification}. Many approaches focus on accurate value function estimation~\cite{fujimoto2018addressing,moskovitz2021tactical,wei2022controlling}, exploration term design~\cite{haarnoja2018soft,liu2020provably}, and real-world applications~\cite{delarue2020reinforcement,chen2022reinforcement}. However, in the training process, the \emph{i.i.d} assumption of data~\cite{judah2014active,chen2019information,zhang2021sample,liu2022understanding} makes them often overlook the inherent domain knowledge~\cite{deramo2020sharing}, which is available as an offline optimal policy by re-stitching different state-action pairs~\cite{xu2022policy} in the replay buffer. This oversight leads to issues of sample inefficiency and training instability, which OBAC addresses by leveraging this knowledge to enhance policy performance.

\noindent\textbf{Offline RL.}\quad
Given a fixed dataset without environmental interaction, offline RL aims to train a policy within the support of the training distribution~\cite{levine2020offline,prudencio2023survey}, categorized by either explicit or implicit constraints. Explicit constraints involve learning an empirical distribution of behavior policy~\cite{fujimoto2019off,kumar2019stabilizing} for policy regularization or improving in-distribution action values while decreasing out-of-distribution (OOD) actions' values~\cite{kumar2020conservative}. Implicit constraints~\cite{kostrikov2021offline,xiao2022sample,mao2024odice} achieve similar value/policy regularization without additional behavior approximation through implicit constraints. However, the pessimistic principle in these methods makes their application challenging in online settings~\cite{xie2021policy}, as the policy tends to be too conservative to explore better actions. While, our work adopts offline RL techniques within an online training setting, where the data is growing as the policy training and exploring, avoiding the conservative nature of offline RL.

\noindent\textbf{Bridging off-policy and offline RL.}\quad
Recent works have explored the settings of offline-to-online RL~\cite{lee2022offline,yu2023actor,zhang2022policy} or Hybrid RL~\cite{niu2022trust,panaganti2022robust}, leveraging additional offline datasets for policy pretraining~\cite{zhang2022policy,yu2023actor} or for training data augmentation~\cite{wagenmaker2023leveraging,ball2023efficient,song2022hybrid,uchendu2023jump} to fine-tune the learning policy and improve sample efficiency. In contrast, OBAC's training can start with zero data, following a pure online setting without pretraining. Some similar methods introduce the offline optimal $Q$-value, estimated by the transitions observed in the replay memory, to regularize the $Q$-value of the online learning policy~\cite{zhang2022replay,ji2023seizing}. However, our motivating examples demonstrate that offline optimal $Q$-value may be suboptimal to the learning one, especially in the early training stage, leading to a potential drawback for the learning policy. OBAC addresses this issue by introducing an adaptive constraint that smartly determines the timing of introducing the offline optimal policy.

\section{Conclusion}
In this work, we propose a novel off-policy RL framework OBAC, where an agent can exploit an offline optimal policy concurrently trained by the replay buffer to boost the performance of the online learning policy. Based on the theoretical results of offline-boosted policy iteration, which shows the convergence to the optimal policy, this naturally derives a practical algorithm with low computational cost and high sample efficiency. Abundant experimental results demonstrate the superiority of OBAC when compared with both model-free and model-based RL baselines. Our findings offer valuable insights 
that leveraging collected data to derive an offline optimal policy can effectively improve the sample efficiency, introducing a novel and practical approach to combining off-policy RL and offline RL. Future works of OBAC can be extended by applying more advanced offline RL methods when deriving the offline optimal policy, or adding more exploration in the online policy to facilitate exploring better data, thus allowing OBAC to improve policy performance from both exploitation and exploration.

% Acknowledgements should only appear in the accepted version.
\section*{Acknowledgements}

This work was done at Tsinghua University and supported by the Xiaomi Innovation Joint Fund of the Beijing Municipal Natural Science Foundation (L233006), partly by the National Natural Science Foundation of China under Grant (U22A2057) and the THU-Bosch JCML Center. We would like to thank Liyuan Mao for his insightful advice on the discussion of Proposition 3.1, and we appreciate the reviewers' generous help to improve our paper.

\section*{Impact Statement}
This work contributes to advancing the field of Reinforcement Learning (RL), particularly in the domain of off-policy RL algorithms. The proposed algorithm holds potential implications for real-world applications, especially in areas such as robotics. As a general off-policy RL algorithm, the main uncertainty of the proposed method might be the fact that the RL training process itself is somewhat brittle. Besides, the exploration of an RL agent in real-world environments may require several safety considerations, to avoid unsafe behavior during the exploration process.
\bibliography{example_paper}
\bibliographystyle{icml2024}

\newpage
\appendix
\onecolumn
\section{Theoretical Analyses}\label{sec:Theoretical_Analyses}
\begin{lemma}\label{Q_converge}
Given a fixed policy $\chi$, the Bellman Expectation Operator $\mathcal{T}^\chi$ is a $\gamma$-contraction mapping within a single evaluation step.
\end{lemma}
\begin{proof}
Let $Q_1(s,a)$ and $Q_2(s,a)$ be two arbitrary state-action value functions. Based on the definition of $\mathcal{T}^\chi$, we have
\begin{align}
\Vert \mathcal{T}^\chi Q_1(s,a) - \mathcal{T}^\chi Q_2(s,a)\Vert_\infty&=\Vert r(s,a)-\gamma\mathbb{E}_{s',a'\sim\chi}[Q_1(s',a')]-r(s,a)+\gamma\mathbb{E}_{s',a'\sim\chi}[Q_2(s',a')]\Vert_\infty\nonumber\\
&\leq \gamma\mathbb{E}_{s'}\vert \mathbb{E}_{a'\sim\chi}[Q_1(s',a')]-\mathbb{E}_{a'\sim\chi}[Q_2(s',a')] \vert\nonumber\\
&\leq \gamma\mathbb{E}_{s',a'\sim\chi}\vert Q_1(s',a') - Q_2(s',a')\vert\nonumber\\
&\leq \gamma\max_{s,a}\vert Q_1(s,a) - Q_2(s,a)\vert\nonumber\\
&=\gamma \Vert Q_1(s,a) - Q_2(s,a)\Vert_{\infty}.
\end{align}
Thus, we conclude that $\mathcal{T}^\chi$ is a $\gamma$-contraction mapping. Further, this property guarantees that given any fixed policy $\chi$, any initial $Q$ function will converge to a unique fixed point by repeatedly applying this operator.
\end{proof}

Recall the constrained optimization problem in policy improvement step, 
\begin{align*}
&\quad\quad\quad\quad\quad\pi_{k+1}=\arg\max_\pi\mathbb{E}_{a\sim\pi}[Q^{\pi_k}(s,a)]\\
&\text{s.t.}\ \int_{a\in\mathcal{A}}f\left(\frac{\pi(a|s)}{\mu^*_k(a|s)}\right)\mathbbm{1}\left(V^{\mu^*_k}(s)-V^{\pi_k}(s)\right)\mu^*_k(a|s)da\leq\epsilon,\label{eq:adative_constraint}\\
& \quad\quad\quad\quad\quad\ \int_{a\in\mathcal{A}}\pi(a|s) da=1, \ \forall s\in\mathcal{D},
\end{align*}
\begin{proposition}\label{closed_form_appendix}
For the constrained optimization problem defined by~(\ref{eq:optimization_problem})$\sim$(\ref{eq:policy_norm}), if $V^{\mu^*_k}(s)\geq V^{\pi_k}(s)$, the closed-form solution is
\begin{equation}
\pi_{k+1}=\frac{1}{Z(s)}\mu^*_k(a|s)\left(f'\right)^{-1}\Big(Q^{\pi_k}(s,a)\Big),
\end{equation}
where $Z(s)$ is a partition function to normalise the action distribution. Or, when $V^{\mu^*_k}(s)<V^{\pi_k}(s)$, $\pi_{k+1}$ is an ordinary solution to maximize $Q^{\pi_k}(s,a)$.
\end{proposition}
\begin{proof}
Following prior methods~\cite{peters2010relative,peng2019advantage}, we apply the KKT conditions for the constrained optimization problem. The Lagrangian is:
\begin{equation}
\mathcal{L}(\pi,\lambda,\alpha)=\mathbb{E}_{a\sim\pi}[Q^{\pi_k}(s,a)]+\lambda\left[\epsilon-f\left(\frac{\pi(a|s)}{\mu^*(a|s)}\right)\mathbbm{1}\left(V^{\mu^*_k}(s)-V^{\pi_k}(s)\right)\mu^*_k(a|s)\right]+\alpha\left(1-\int_{a\in\mathcal{A}}\pi(a|s)da\right).
\end{equation}
Then, we perform differentiation with respect to $\pi$, and have
\begin{align}
\frac{\partial \mathcal{L}}{\partial \pi}&=Q^{\pi_k}(s,a)-\lambda\bcancel{\mu^*_k(a|s)}\frac{\mathbbm{1}\left(V^{\mu^*_k}(s)-V^{\pi_k}(s)\right)}{\bcancel{\mu^*_k(a|s)}}f'\left(\frac{\pi(a|s)}{\mu^*(a|s)}\right)-\alpha \nonumber\\
&=Q^{\pi_k}(s,a)-\lambda\mathbbm{1}\left(V^{\mu^*_k}(s)-V^{\pi_k}(s)\right)f'\left(\frac{\pi(a|s)}{\mu^*(a|s)}\right)-\alpha
\end{align}

By KKT conditions, we set $\frac{\partial \mathcal{L}}{\partial \pi}=0$. When $V^{\mu^*_k}(s)\geq V^{\pi_k}(s)$, \emph{i.e.}, $\mathbbm{1}\left(V^{\mu^*_k}(s)-V^{\pi_k}(s)\right)=1$, then we have
\begin{equation}
\pi_{k+1}=\frac{1}{Z(s)}\mu^*(a|s)\left(f'\right)^{-1}\Big(Q^{\pi_k}(s,a)\Big),
\end{equation}
where $Z(s)$ is a partition function to normalise the action distribution, and $\lambda$ is a behavior clone weight.

In contrast, if $V^{\mu^*_k}(s)< V^{\pi_k}(s)$, the constraint~(\ref{eq:adative_constraint}) is ineffective. Thus, $\pi_{k+1}$ is an ordinary solution to maximize $Q^{\pi_k}(s,a)$. The proof is completed.
\end{proof}

\clearpage
\begin{proposition}\label{Q_improvement}
Let $\pi_k$ be the older learning policy and the newer one $\pi_{k+1}$ be the solution of~(\ref{eq:optimization_problem})$\sim$(\ref{eq:policy_norm}). Then we achieve $Q^{\pi_{k+1}}(s,a)\geq Q^{\pi_k}(s,a)$ for all $(s,a)\in\mathcal{D}$, with the offline optimal policy $\mu^*_k$ serving as a performance baseline policy.
\end{proposition}
\begin{proof}
Since $\pi_{k+1}$ is the solution of~(\ref{eq:optimization_problem})$\sim$(\ref{eq:policy_norm}), we discuss it into two cases:
\begin{itemize}
    \item Unconstrained Optimization Problem
\end{itemize}
In this case, we have $\pi_{k+1}=\arg\max_{\pi}\mathbb{E}_{a\sim\pi}[Q^{\pi_k}(s,a)]$. Thus, it satisfies $\mathbb{E}_{a\sim\pi_{k+1}}[Q^{\pi_k}(s,a)]\geq\mathbb{E}_{a\sim\pi_k}[Q^{\pi_k}(s,a)]$. In a similar way to the proof of the soft policy improvement~\citep{haarnoja2018soft}, we come to the following inequality:
\begin{align}
Q^{\pi_k}(s_t,a_t)&=r(s_t,a_t)+\gamma\mathbb{E}_{s_{t+1},a_{t+1}\sim\pi_k}[Q^{\pi_k}(s_{t+1},a_{t+1})]\nonumber\\
&\leq r(s_t,a_t)+\gamma\mathbb{E}_{s_{t+1},a_{t+1}\sim\pi_{k+1}}[Q^{\pi_k}(s_{t+1},a_{t+1})]\nonumber\\
& \quad\vdots\nonumber\\
&=Q^{\pi_{k+1}}(s_t,a_t)
\end{align}
Thus, we can obtain that $Q^{\pi_{k+1}}(s,a)\geq Q^{\pi_k}(s,a)$.

\begin{itemize}
    \item Constrained Optimization Problem
\end{itemize}
Based on Proposition~\ref{closed_form_appendix}, we have the closed-form solution of $\pi_{k+1}$ as
\begin{equation*}
\pi_{k+1}=\frac{1}{Z(s)}\mu^*(a|s)\left(f'\right)^{-1}\Big(Q^{\pi_k}(s,a)\Big).
\end{equation*}
Note that the definition of the offline optimal policy is $\mu^*_k(a|s)=\arg\max_{a\sim\mathcal{D}} Q^{\mu^*_k}(s,a)$. Then, when the constraint is effective, \emph{i.e.}, $V^{\mu^*_k}(s)\geq V^{\pi_k}(s)$, we can derive that
\begin{align}\label{eq_15}
\mathbb{E}_{a\sim\pi_{k+1}}[Q^{\pi_k}(s,a)]&=\int_{a\sim\mathcal{A}}\pi_{k+1}(a|s)Q^{\pi_k}(s,a)da\nonumber\\
&=\int_{a\sim\mathcal{A}}\frac{1}{Z(s)}\mu^*(a|s)\left(f'\right)^{-1}\Big(Q^{\pi_k}(s,a)\Big)Q^{\pi_k}(s,a)da\nonumber\\
&=\int_{a\sim\mathcal{A}}\frac{1}{Z(s)}\left[\arg\max_{a\sim\mathcal{D}} Q^{\mu^*_k}(s,a)\right]\left(f'\right)^{-1}\Big(Q^{\pi_k}(s,a)\Big)Q^{\pi_k}(s,a)da\nonumber\\
&\geq \int_{a\sim\mathcal{A}}\frac{1}{Z(s)}\left[\arg\max_{a\sim\mathcal{D}} Q^{\pi_k}(s,a)\right]\left(f'\right)^{-1}\Big(Q^{\pi_k}(s,a)\Big)Q^{\pi_k}(s,a)da\nonumber\\
&\geq\int_{a\sim\mathcal{A}}\frac{1}{Z(s)}\pi_k(a|s)\left(f'\right)^{-1}\Big(Q^{\pi_k}(s,a)\Big)Q^{\pi_k}(s,a)da\quad\quad \lhd\ \text{Satisfied in $\mathcal{D}$}\nonumber\\
&\geq \int_{a\sim\mathcal{A}}\pi_k(a|s)Q^{\pi_k}(s,a)da\nonumber\\
&=\mathbb{E}_{a\sim\pi_k}[Q^{\pi_k}(s,a)].
\end{align}

Thus, we reuse the results in the unconstrained optimization problem, we can have $Q^{\pi_{k+1}}(s,a)\geq Q^{\pi_k}(s,a)$.

Combining the results in both cases, we achieve $Q^{\pi_{k+1}}(s,a)\geq Q^{\pi_k}(s,a)$ for all $(s,a)\in\mathcal{D}$.
\end{proof}

\begin{proposition}
Assume $\vert \mathcal{A}\vert < \infty$, repeating the alternation of the policy evaluation~(\ref{evalua_for_Q})$\sim$(\ref{V_evalua}) and policy improvement~(\ref{eq:optimization_problem})$\sim$(\ref{eq:policy_norm}) can make any online learning policy $\pi_k\in\Pi$ converge to the optimal policies $\pi^*$, s.t. $Q^{\pi^*}(s_t,a_t)\geq Q^{\pi_k}(s_t,a_t), \forall (s_t,a_t) \in \mathcal{S}\times \mathcal{A}$.
\end{proposition}
\begin{proof}
Suppose $\Pi$ is the policy set and $\pi_k$ is the policy at iteration $k$. At each iteration, we guarantee the sequence $Q^{\pi_k}$ is monotonically increasing through Proposition~\ref{Q_improvement}. Besides, $\forall (s_t,a_t) \in \mathcal{S}\times \mathcal{A}$, $Q^{\pi_i}$ would converge by repeatedly using the Bellman Expectation Equation as a $\gamma$-contraction mapping, which is proved in Lemma~\ref{Q_converge}. Thus, the sequence of $\pi_k$ converges to some $\pi^*$ that are local optimum. Then, we would show that $\pi^*$ is indeed optimal. Using the same iterative argument as in the proof of Proposition~\ref{Q_improvement}, the optimal policy $\pi^*$ would satisfy that $Q^{\pi^*}(\rvs,\rva) \geq Q^{\pi'}(s,a)$ for all $(s,a)\in \mathcal{S}\times \mathcal{A}$. Hence, $\pi^*$ are optimal in $\Pi$.
\end{proof}

\clearpage
\section{Implementation Details}\label{sec:implementation}
In this section, we delve into the specific implementation details of OBAC. We use the same hyperparameters for all OBAC experiments in this paper. In terms of architecture, we use a simple 2-layer ELU network with a hidden size of 512 to parameterize all components, which contains: a $Q$-value network and a policy network for the online learning policy $\pi$; a $Q$-value network and a $V$-value network for the offline optimal policy $\mu^*$. 

Specifically, to encourage online learning policy exploration, we utilize a max-entropy framework~\cite{haarnoja2018softappli} for $\pi$ with automatic temperature tuning.
\begin{table}[ht]
\caption{The hyperparameters of OBAC}
\label{table:hypermeter}
\centering
\begin{tabular}{@{}lll@{}}
\toprule
\multicolumn{1}{c}{\multirow{11}{*}{\centering OBAC Hyperparameters}} & Hyperparameter                    & Value     \\ \cmidrule(lr){2-3}
                                                                    & Optimizer                         & Adam      \\
                                                                    & Critic learning rate              & 3e-4      \\
                                                                    & Actor learning rate               & 3e-4      \\
                                                                    & Discount factor                   & 0.99      \\
                                                                    & Mini-batch                        & 512       \\
                                                                    & Actor Log Std. Clipping           & $(-20,2)$ \\
                                                                    & Replay buffer size                & 1e6
                                                                    \\ 
                                                                    & Expectile factor $\tau$          & 0.9
                                                                    \\ 
                                                                    & Behavior clone weight $\lambda$          & 0.001\\
                                                                    \cmidrule(lr){1-3}
\multicolumn{1}{c}{\multirow{3}{*}{\centering Architecture $\times 4$}}        & Network hidden dim                 & 512       \\
                                                                    & Network  hidden layers              & 2         \\
                                                                    & Network  activation function        & elu       \\
                                                                    \bottomrule
\end{tabular}
\end{table}

\subsection{Hyper-parameters}
In all of our experiments, we use a single set of hyper-parameters with $\tau=0.9$ and $\lambda=0.001$. Here, a big expectile factor $\tau$ can approach the maximization of offline $Q$ value in the replay buffer, thus enabling better online policy learning. To balance the training stability and optimality of offline policy, we choose $\tau=0.9$. For the behavior clone weight $\lambda$ which is from the solution within the KKT condition, we can apply dual gradient descent for auto-tuning in principle, while we find a fixed $\lambda=0.001$ can achieve satisfied performance. Besides, several offline RL works~\cite{kumar2019stabilizing,fujimoto2021minimalist} also use fixed weight for policy constraint. Thus, we apply a fixed $\lambda$ in our works.

\subsection{Baselines and Environments}
In our experiments, we have implemented SAC, TD3 and TD-MPC2 using their original code bases to ensure a fair and consistent comparison.
\begin{itemize}[left=5pt]
    \item For SAC~\citep{haarnoja2018soft}, we utilized the open-source PyTorch implementation, available at \url{https://github.com/pranz24/pytorch-soft-actor-critic}.
    \item TD3~\citep{fujimoto2018addressing} was integrated into our experiments through its official codebase, accessible at \url{https://github.com/sfujim/TD3}.
    \item TD-MPC2~\citep{hansen2023td} was employed with its official implementation from \url{https://github.com/nicklashansen/tdmpc2}.
\end{itemize}

For BAC~\cite{ji2023seizing}, we reproduce the proposed BEE operator. Specifically, the Bellman Exploitation operator $\mathcal{T}^\mu_{exploit}$ is implemented by IQL~\cite{kostrikov2021offline}, and the Bellman Exploration operator $\mathcal{T}^\pi_{explore}$ with the entropy exploration term is based on SAC~\cite{haarnoja2018soft}, both of which are suggested in its original paper. We choose the trade-off hyper-parameter in BAC as $0.5$ and expectile factor $0.7$ aligning with its most suggestion.

We use the official setting of each task domain, including the reward setting, the task horizon, the \textit{done} signal and the original state-action spaces.

\clearpage
\section{More Experimental Results}
We provide complete experimental results to show the superiority of OBAC.

\subsection{Task Visualization}
\begin{figure}[htbp]
    \centering
    \includegraphics[width=\textwidth]{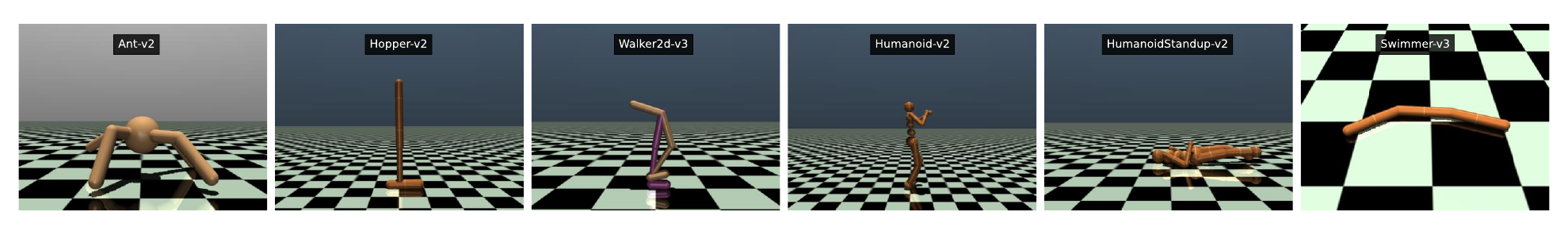}
    \caption{Visualization of tasks in \textbf{Mujoco}.}
    \label{fig:mujoco_visualition}
\end{figure}
\begin{figure}[htbp]
    \centering
    \includegraphics[width=\textwidth]{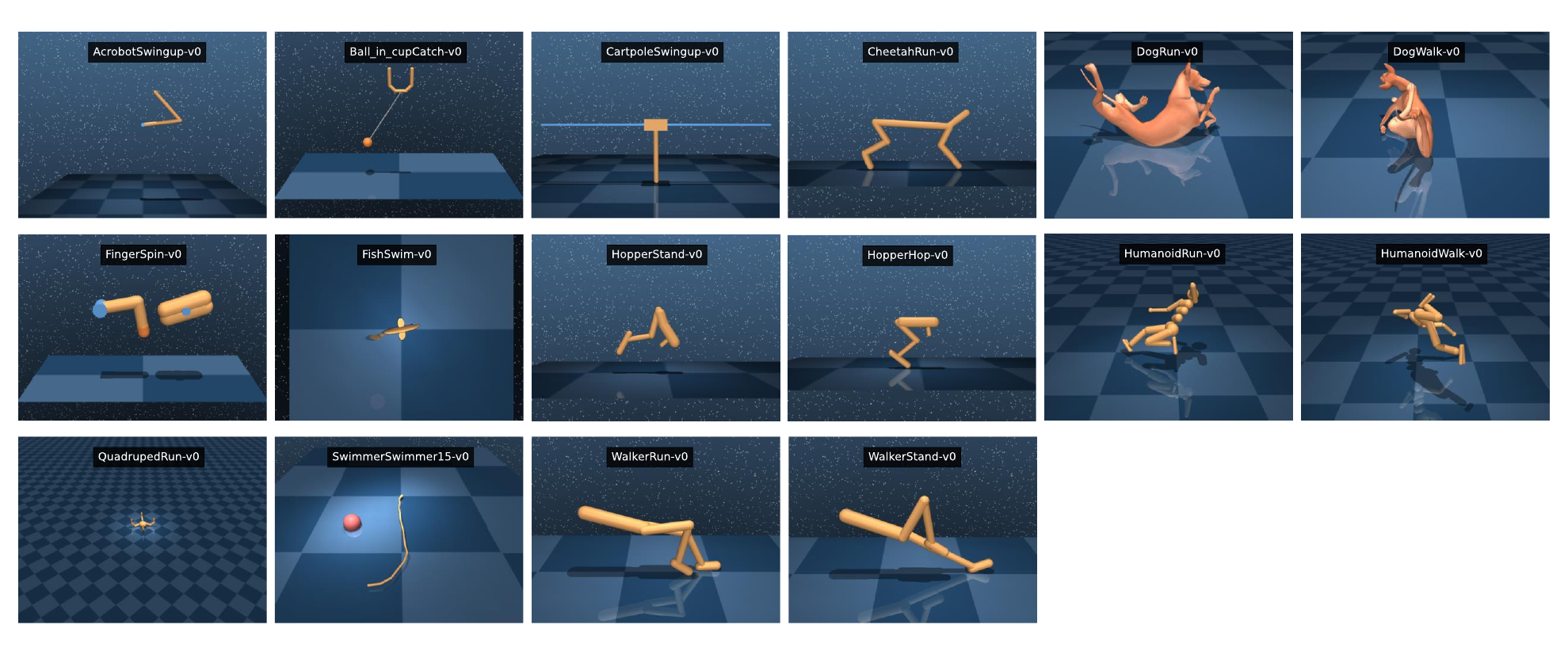}
    \caption{Visualization of tasks in \textbf{DM Control}.}
    \label{fig:dmc_visualition}
\end{figure}
\begin{figure}[htbp]
    \centering
    \includegraphics[width=\textwidth]{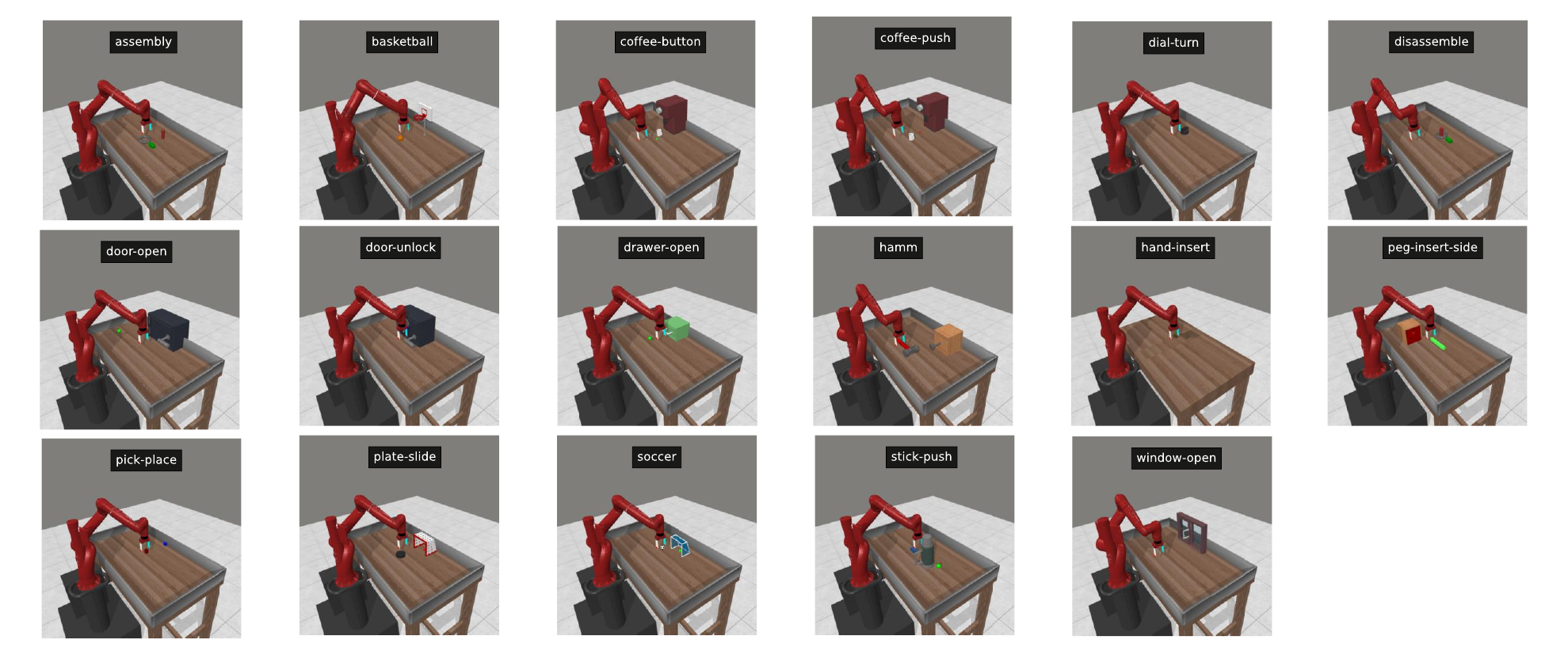}
    \caption{Visualization of tasks in \textbf{Meta-World}.}
    \label{fig:meta_world_visualition}
\end{figure}
\begin{figure}[htbp]
    \centering
    \includegraphics[width=\textwidth]{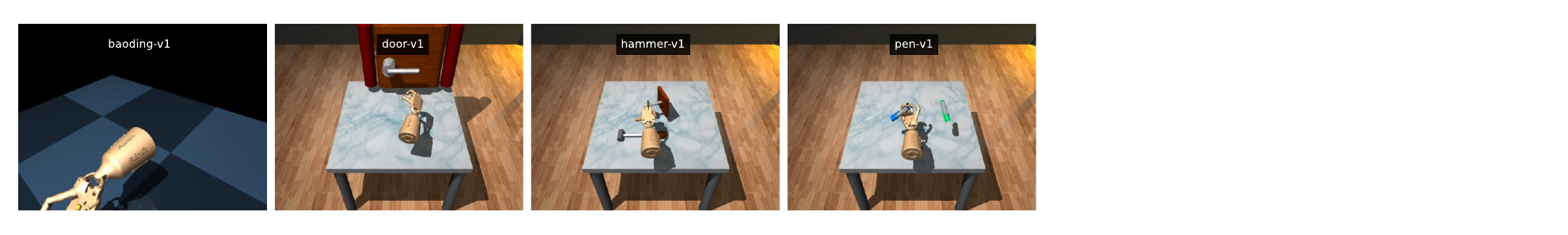}
    \caption{Visualization of tasks in \textbf{Adroit}.}
    \label{fig:adroit_visualition}
\end{figure}
\begin{figure}[htbp]
    \centering
    \includegraphics[width=\textwidth]{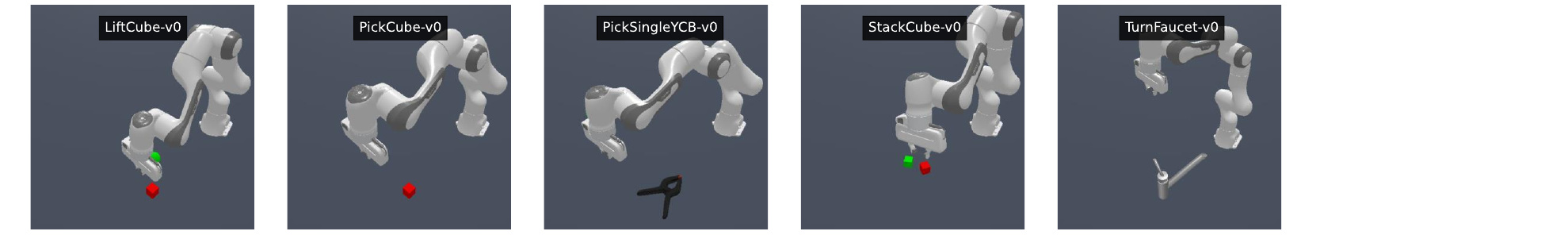}
    \caption{Visualization of tasks in \textbf{Maniskill2}.}
    \label{fig:maniskill_visualition}
\end{figure}
\begin{figure}[htbp]
    \centering
    \includegraphics[width=\textwidth]{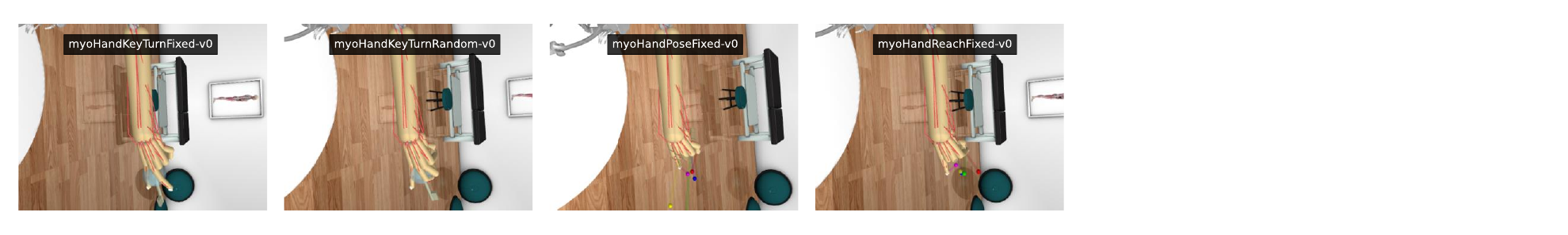}
    \caption{Visualization of tasks in \textbf{Myosuite}.}
    \label{fig:myosuite_visualition}
\end{figure}

\clearpage
\subsection{Complete Experimental Results}\label{sec:all_results}
We note that, since the task-specific \textit{done} signal setting, TD-MPC2 may perform poorly in Mujoco suite, especially for Ant, Hopper, Walker2d and Humanoid, where the episode may be terminated early if the robot falls. In such cases, TD-MPC2 can not find a feasible policy to prevent the fall thus achieving limited performance. Besides, we did not find the results of Mujoco in TD-MPC2's paper.

However, in the other suites, we find TD-MPC2 can perform well even within unseen tasks in its original paper. Thus, we think the reproduction results are reasonable. Note that, in these suites, the episode would be done only when the task horizon comes to an end, which may provide more exploration information compared with Mujoco suite.

For the consideration of fair comparison, we follow the official setting of each task suite when evaluating all algorithms.
\begin{figure}[htbp]
    \centering
    \includegraphics[width=\textwidth]{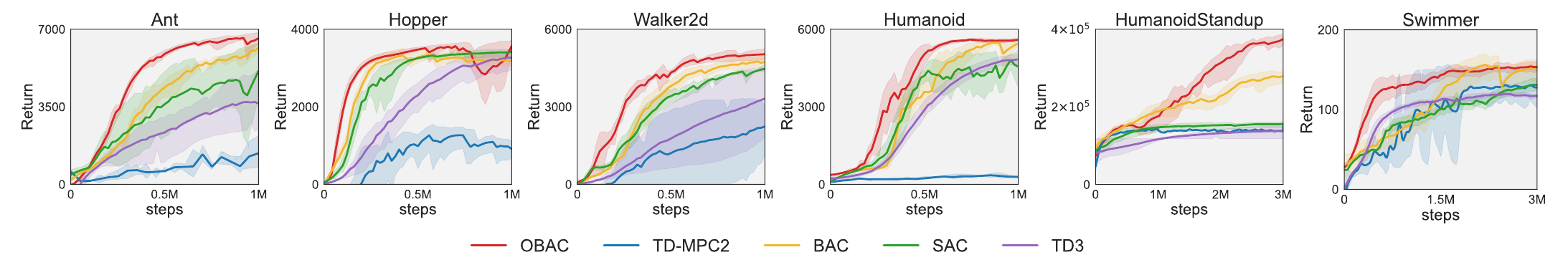}
    \caption{The results of \textbf{6} tasks in \textbf{Mujoco}.}
    \label{fig:mujoco_results}
\end{figure}
\begin{figure}[htbp]
    \centering
    \includegraphics[width=\textwidth]{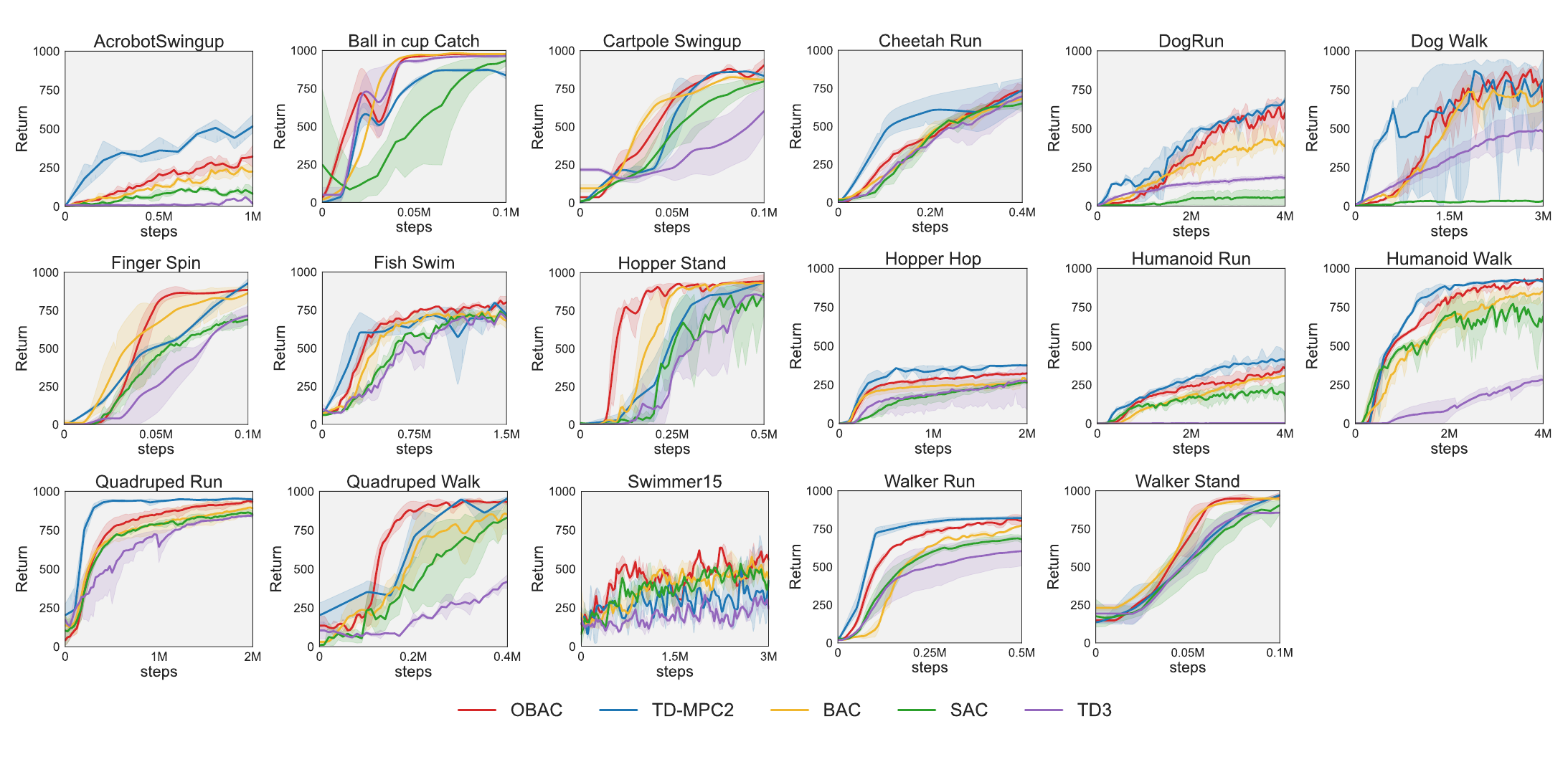}
    \caption{The results of \textbf{17} tasks in \textbf{DM Control}.}
    \label{fig:dmc_results}
\end{figure}
\begin{figure}[htbp]
    \centering
    \includegraphics[width=\textwidth]{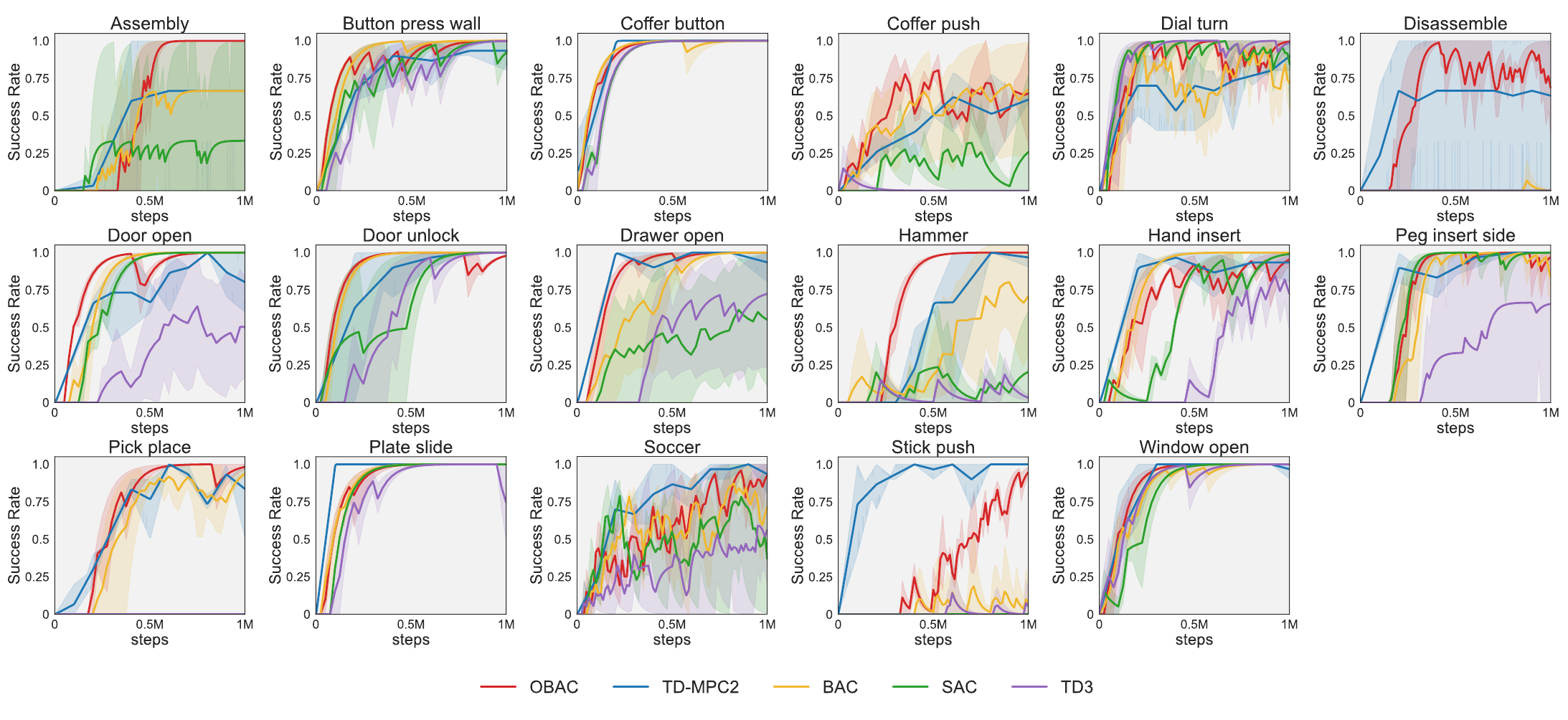}
    \caption{The results of \textbf{17} tasks in \textbf{Meta-World}.}
    \label{fig:metaworld_results}
\end{figure}
\begin{figure}[htbp]
    \centering
    \includegraphics[width=\textwidth]{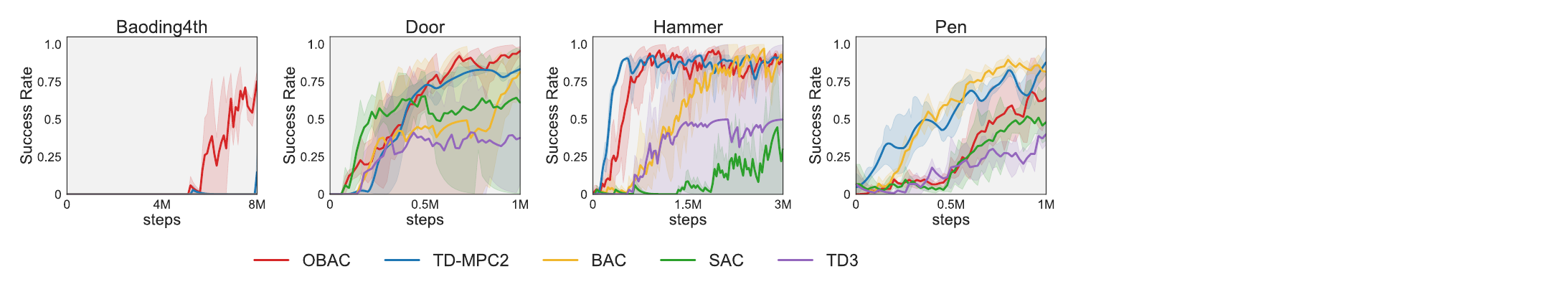}
    \caption{The results of \textbf{4} tasks in \textbf{Adroit}.}
    \label{fig:adroit_results}
\end{figure}
\begin{figure}[htbp]
    \centering
    \includegraphics[width=\textwidth]{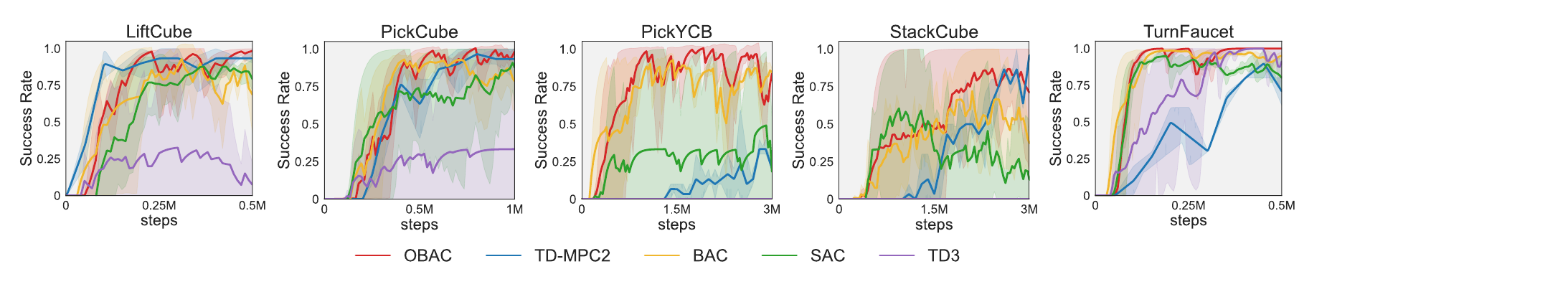}
    \caption{The results of \textbf{5} tasks in \textbf{Maniskill2}.}
    \label{fig:maniskill_results}
\end{figure}
\begin{figure}[htbp]
    \centering
    \includegraphics[width=\textwidth]{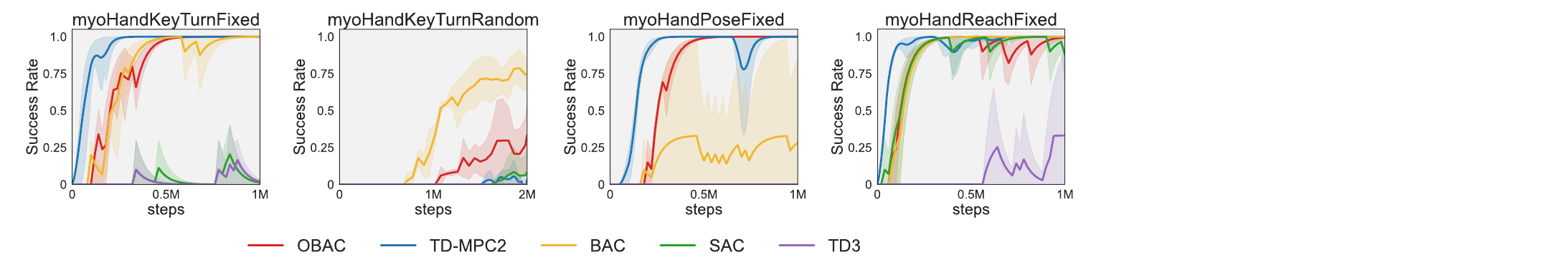}
    \caption{The results of \textbf{4} tasks in \textbf{Myosuite}.}
    \label{fig:myosuite_results}
\end{figure}

\clearpage
\subsection{Performance comparison under noise tasks}\label{appendix_noise}
To better explain the robustness of OBAC, we provide the performance decline rates of OBAC and baselines here, whose results are similar to Figure~\ref{fig:noise_ablation}. Our results show that OBAC exhibits a lower performance decline than the baselines.
\begin{itemize}
    \item Given the same noise level ($\sigma=0.1$)
    \begin{table}[H]
        \caption{Performance comparison of OBAC and baselines with the same noise level}
        \centering
        \begin{tabular}{cccc}
        \toprule
        \textbf{Task (success rate)}& \textbf{OBAC} & \textbf{BAC} & \textbf{SAC} \\
        \midrule 
        Hammer ($\sigma=0$) & 0.9 & 0.9 & 0.4 \\
        Hammer ($\sigma=0.1$) & 0.8 & 0.75 & 0.35 \\
        \textbf{Decline rate} & $\mathbf{11.11\%}$ & $16.67\%$ & $12.5\%$ \\
        \midrule
        Pen ($\sigma=0$) & 0.6 & 0.75 & 0.5 \\
        Pen ($\sigma=0.1$) & 0.45 & 0.25 & 0.15 \\
        \textbf{Decline rate} & $\mathbf{25.00\%}$ & $66.67\%$ & $70.00\%$ \\
        \bottomrule
        \end{tabular}
    \end{table}
    \item Given the different noise level ($\sigma=0.1$ and $\sigma=0.05$)
    \begin{table}[H]
        \caption{Performance comparison of OBAC with different noise levels}
        \centering
        \begin{tabular}{cc}
        \toprule
        \textbf{Task (success rate)}& \textbf{OBAC} \\
        \midrule 
        Hammer ($\sigma=0.05$) & 0.85 \\
        Hammer ($\sigma=0.1$) & 0.8 \\
        \textbf{Decline rate} & $\mathbf{5.88\%}$ \\
        \midrule 
        Pen ($\sigma=0.05$) & 0.5  \\
        Pen ($\sigma=0.1$) & 0.45 \\
        \textbf{Decline rate} & $\mathbf{10.00\%}$ \\
        \bottomrule
        \end{tabular}
    \end{table}
\end{itemize}

\clearpage
\section{Additional Ablations}\label{sec:more_ablation}
Except for the ablation studies in the main paper, we additionally provide the results of OBAC's variant to assess OBAC completely.

\paragraph{Variant of OBAC.}
In our implementation, we employ the stochastic Gaussian policy for online policy learning. On the other side, we derive a variant of OBAC by using a deterministic policy for the online learning policy. Thus, the policy objective can be
\begin{align}\label{eq:update_pi_determin}
\arg\min_{\pi_\theta\in\Pi}\mathbb{E}_{s\sim\mathcal{D}}\Big\{\mathbb{E}_{a\sim\pi_\theta}[-Q^{\pi_k}(s,a)]-\lambda\mathbb{E}_{a\sim\mathcal{D}}\left[{\color{blue}(\pi_\theta(s)-a)^2}\right]\mathbbm{1}\left(V^{\mu^*_k}(s)-V^{\pi_k}(s)\right)\Big\}.
\end{align}

We conduct several experiments on such a variant. The results show that our method can also 
\begin{figure}[htbp]
    \centering
    \includegraphics[width=0.25\textwidth]{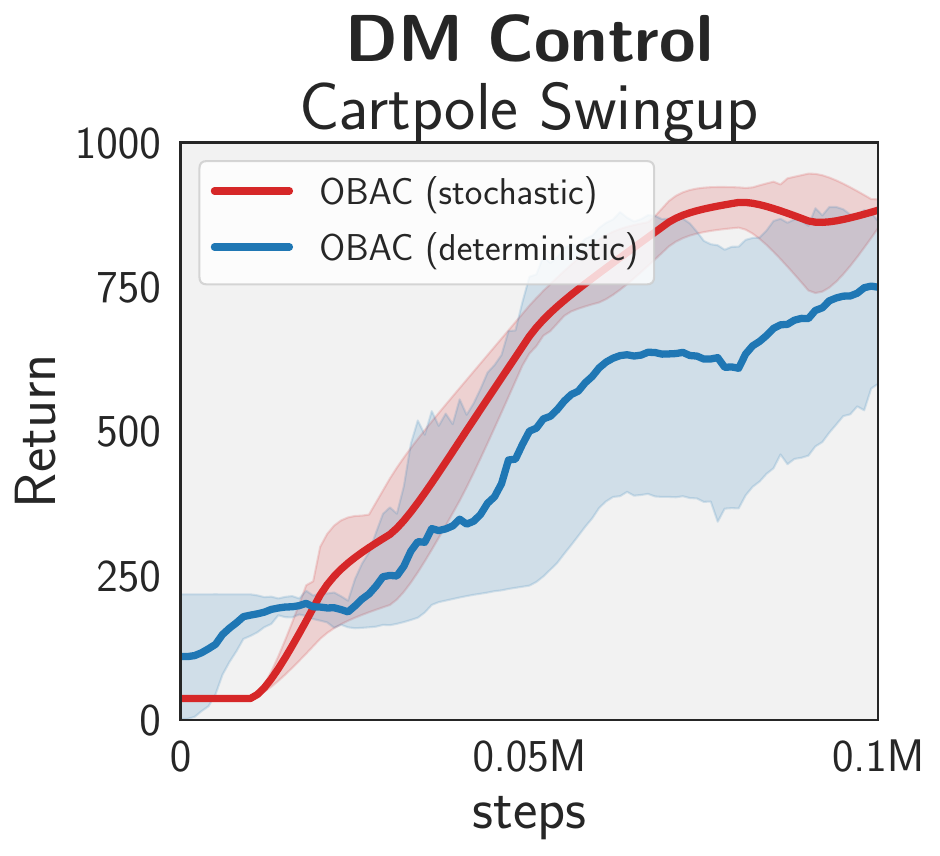}
    \includegraphics[width=0.25\textwidth]{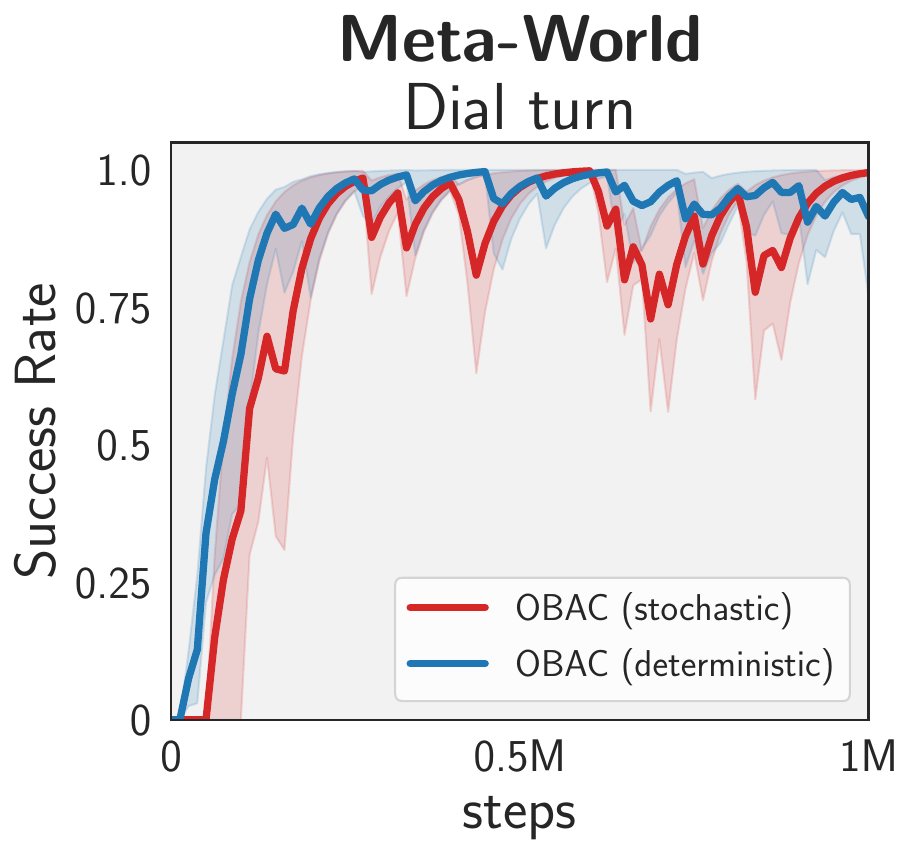}
    \includegraphics[width=0.25\textwidth]{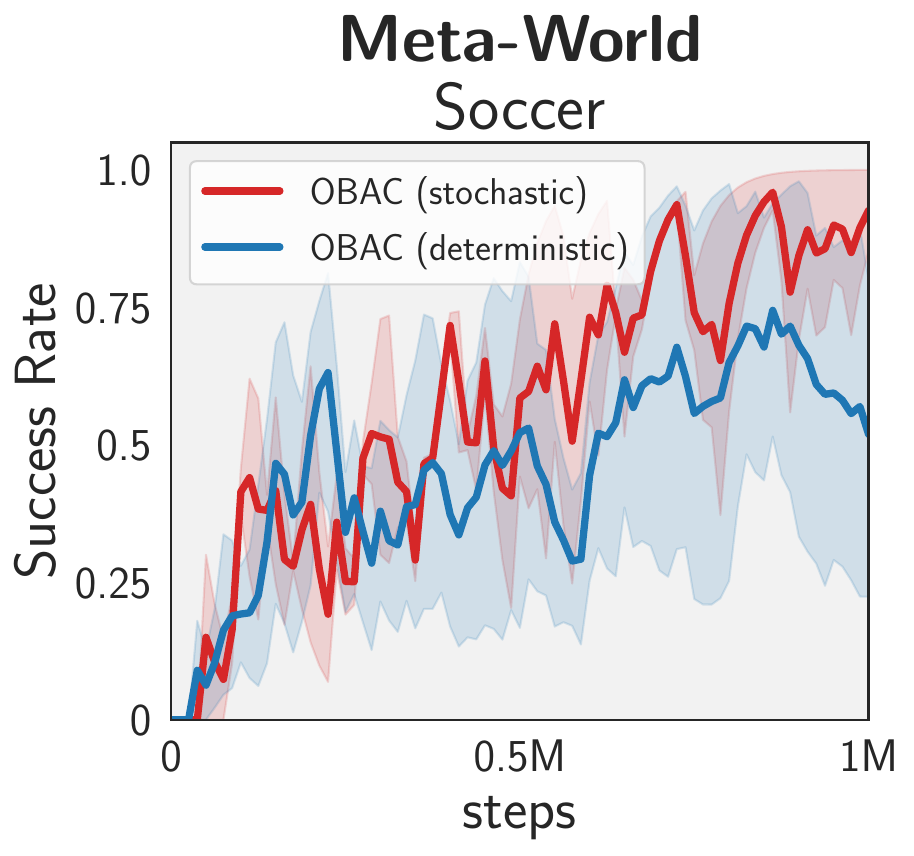}
    \caption{Comparison of OBAC and its variant.}
    \label{fig:variant_results}
\end{figure}

\end{document}